\documentclass[nohyperref]{article}

\usepackage{microtype}
\usepackage{graphicx}
\usepackage{subfigure}
\usepackage{booktabs, multirow} 
\usepackage{hyperref}

\usepackage[accepted]{icml2023}

\usepackage{amsmath}
\usepackage{amssymb}
\usepackage{mathtools}
\usepackage{amsthm}
\usepackage{bm}
\usepackage{bbm}
\usepackage{comment}

\newcommand{\indep}{\perp \!\!\! \perp}

\usepackage[capitalize,noabbrev]{cleveref}

\theoremstyle{plain}
\newtheorem{theorem}{Theorem}[section]
\newtheorem{proposition}[theorem]{Proposition}
\newtheorem{lemma}[theorem]{Lemma}

\theoremstyle{definition}
\newtheorem{definition}{Definition}[section]

\newtheorem{assumpV}{Assumption}
\newtheorem{assumpM}{Assumption}

\theoremstyle{remark}

\usepackage[textsize=tiny]{todonotes}

\icmltitlerunning{Covariate balancing using the integral probability metric}

\begin{document}

\twocolumn[
\icmltitle{Covariate balancing using the integral probability metric for causal inference}



\icmlsetsymbol{equal}{*}

\begin{icmlauthorlist}
\icmlauthor{Insung Kong}{snu}
\icmlauthor{Yuha Park}{snu}
\icmlauthor{Joonhyuk Jung}{snu}
\icmlauthor{Kwonsang Lee}{snu}
\icmlauthor{Yongdai Kim}{snu}
\end{icmlauthorlist}

\icmlaffiliation{snu}{Department of Statistics, Seoul National University}

\icmlcorrespondingauthor{Yongdai Kim}{ydkim0903@gmail.com}

\icmlkeywords{Machine Learning, ICML}

\vskip 0.3in
]



\printAffiliationsAndNotice{}  

\begin{abstract}
Weighting methods in causal inference have been widely used to achieve a desirable level of covariate balancing. However, the existing weighting methods have desirable theoretical properties
only when a certain model, either the propensity score or outcome regression model,
is correctly specified. In addition, the corresponding estimators
do not behave well for finite samples due to large variance even when the model is correctly specified.
In this paper, we consider to use the
integral probability metric (IPM), which is a metric between two probability measures, for covariate balancing. Optimal weights are determined so that weighted empirical distributions for the treated and control groups have the smallest IPM value for a given set of discriminators. 
We prove that the corresponding estimator can be consistent without correctly specifying any model (neither the propensity score nor the outcome regression model).
In addition, we empirically show that our proposed method outperforms
existing weighting methods with large margins for finite samples. 
  

\end{abstract}

\section{Introduction}
Estimating causal effects from observational data has become an important research topic since it has the advantage of low budget requirements and a large amount of available data, compared with randomized trials \cite{yao2021survey}. 
The main difficulty in estimating the causal effect is that, for a binary treatment, the pre-treatment covariate distribution of the treated group differs significantly from that of the control group. This systematic difference requires sophisticated methods to find a causal relationship from observational data \cite{lunceford2004stratification, imbens2015causal}. There have been lots of methods developed including regression adjustment \cite{hill2011bayesian}, matching \cite{stuart2010matching} and weighting \cite{imai2014covariate} methods to adjust for the systematic bias. The main goal of these methods is to make the two groups comparable, meaning that their covariate distributions are (asymptotically) balanced. Achieving the covariate balancing is essential to causal inference.




One of the core approaches in covariate balancing is inverse probability weighting (IPW) \cite{hirano2003efficient}.
The IPW method gives weights to observed samples reciprocally proportional to the propensity scores that are the conditional probabilities of being assigned to the treatment. Using these weights, the IPW can match the covariate distributions of the treated and control groups. However, in practice, the propensity scores (thus weights) are unknown and hence must be estimated.
Propensity scores are conventionally estimated by standard regression models including the logistic regression \cite{lunceford2004stratification} or machine learning techniques \cite{lee2010improving}.

One of the issues using the IPW is that though the weighting scheme eventually achieves the covariate balancing as the sample size increases, there is no guarantee to achieve the balance for given data.
Since the IPW uses the inverse of probabilities, a small error in estimating these probabilities would cause a substantial error in the estimated causal effect \cite{li2018balancing}.
That is, the performance of the IPW estimator highly relies on the correctness of the propensity score estimation
\cite{hainmueller2012entropy, imai2014covariate}.

Instead of estimating the propensity score, \citet{hainmueller2012entropy} and \citet{imai2014covariate} propose to find the weights that match the sample moments of the covariate. By weighting samples to balance the sample moments of the covariate of the two groups, they improve the stability 
of the causal effect estimate to make it more accurate.

The biggest drawback of the existing weighting methods is that their theoretical guarantee and good performance are valid in a restricted situation: either the propensity score or the true outcome model
(i.e. the regression model for either treated or control group)
is a linear function of covariate \cite{sant2022covariate}. 
Without this condition, the estimated causal effect is not guaranteed to be consistent.
To address this issue, they assume that an appropriate transformation exists such that the true outcome regression model is expressed as a linear function of the transformed covariate.
However, since  information required for choosing a good transformation is rarely available in practice, researchers usually apply the methods without transformation.

In this paper, we propose a general framework of covariate balancing  using the integral probability metric (CBIPM). 
The IPM, which includes the Wasserstein distance \cite{kantorovich1958space, Villani2008OptimalTO} as a special case, has been widely used for learning generative models \cite{arjovsky2017wasserstein}, but has not been  popularly used for covaraite balancing. 

The proposed framework is motivated by the fact that the IPM between the covariate distributions of the treated
and control groups is directly related to the worst-case bias of the estimated causal effect with respect to outcome regression models. Thus, by searching the weights that
minimize the IPM of the two covariate distributions, we can reduce the bias to
have a good estimate of causal effect.

We consider two types of algorithms for CBIPM - (1) parametric CBIPM  (P-CBIPM) and (2) nonparametric CBIPM (N-CPIPM), where
the former assumes a certain parametric model for the weights while the later does not. 
An important advantage of the these two CBIPM methods over existing weighting methods
is that the estimated causal effect is consistent under much milder conditions on the true propensity score and outcome regression models. See Section \ref{sec4} for details.

The parametric CBIPM 
has been already used implicitly in estimating the conditional average treatment effect (CATE) or individual treatment effect (ITE) \cite{shalit2017estimating, yao2018representation, wanggeneralization}.
Even though our interest is to improve the existing weighting methods,
useful insights for the CATE or ITE estimation problems can be obtained from the theoretical results in this paper. For example, 
when the parametric model for the weights is correctly specified,
we show that the P-CBIPM estimator is consistent with a minimal set of discriminators.
This result suggests that a simpler set of discriminators is recommended 
when the parametric model for the weights is complex.

The nonparametric CBIPM is a new trial of using the IPM for causal inference,
and yields several important and interesting implications.
For the corresponding ATT estimator to be consistent, the choice of the set of discriminators
in the IPM is important. For example, the ATT estimator is consistent when the true outcome
regression model belongs to the set of discriminators. 
A surprising result, however, is that the ATT estimator can be consistent even when the set of discriminators is fairly small so that it
does not include the complex true outcome regression model.
That is, by the N-CBIPM, we can construct a consistent ATT estimator 
without correctly specifying either the propensity score model or the outcome regression model, which is the first of its kinds.


The main contributions of this work are summarized as
follows.
\begin{itemize}
\item We propose a general framework of covariate
balancing using the IPM to develop two weighting algorithms - parametric CBIPM and nonparametric CBIPM.

\item We prove the consistency of the corresponding estimators of causal effect under
mild regularity conditions.

\item We empirically show that our proposed estimators outperform
existing weighting methods with large margins. 

\end{itemize}

\section{Preliminaries}
\subsection{Notations and Models}
Let $\bm{X} \in \mathcal{X} \subset \mathbb{R}^d$ be a random vector of covariate whose distribution is denoted by $\mathbbm{P}.$
The binary treatment indicator $T$ is generated from $\operatorname{Ber}(\pi(\bm{X}))$, where the propensity score $\pi(\cdot)$ is defined as the conditional probability of receiving the treatment given  covariate.
We assume the strict overlap condition:  there exists $\eta>0$ such that
\begin{align*}
\eta \leq \pi(\bm{x}) \leq 1-\eta,
\end{align*}
for every $\bm{x} \in \mathcal{X}$.
Note that a large body of literature assumes the strict overlap condition as it is
indispensable for theoretical analysis \cite{d2021overlap}.

Let $Y(0)$ and $Y(1)$ denote the potential outcomes under control and treatment, respectively.
We use the ignorability assumption \cite{rosenbaum1983central}: $$T \indep (Y(0) , Y(1)) | \bm{X}.$$
This assumption roughly says that a set of confounders that affect both treatment $T$ and potential outcomes $(Y(0), Y(1))$ is a subset of observable covariates. That is, we do not allow situations where we miss any confounder. When there exists an unmeasured confounder and thus the ignorability assumption is violated,  we typically do sensitivity analysis to evaluate the robustness of the conclusion made under the ignorability assumption \cite{rosenbaum2002overt}. 

Suppose we observe $n$ independent copies $\mathcal{D}^{(n)} = \{(\bm{X}_i, T_i, Y_i)\}_{i=1}^n$,
of $(\bm{X},T,Y),$ where $Y := T Y(1)  + (1-T) Y(0)$ is the observed outcome. Note that we never observe both potential outcomes simultaneously, which is often referred to as the fundamental problem of causal inference \cite{holland1986statistics}. 

The primary goal of this paper is to estimate the average treatment effect for the treated (ATT)
based on $\mathcal{D}^{(n)}.$
The ATT, which is the causal effect of how much the treated units are benefited by the treatment from a retrospective perspective, is defined as
$$ \operatorname{ATT} := \mathbb{E}\left(Y(1)-Y(0)|T=1\right).$$
In turn, the sample ATT (SATT) is defined as
$$\operatorname{SATT} := \frac{1}{n_1}\sum_{i : T_i = 1} \mathbb{E}( Y_i(1) - Y_i(0)).$$
It is easy to show that the SATT converges to the ATT with the converge rate $1/\sqrt{n}$, and thus we focus on estimating the SATT in this paper.

The average treatment effect (ATE) over the population is also popularly considered instead of the ATT,
where the ATE is defined as
$$\operatorname{ATE} := \mathbb{E}\left(Y(1)-Y(0)\right).$$
In this paper, we mainly consider the ATT because of its notational simplicity. But,
the CBIPM methods for the ATT can be easily modified for the ATE, which is discussed in Appendix \ref{app_ATE}.

The true outcome regression models $m_t(\bm{x})=\mathbb{E}(Y (t) |\bm{x}), t\in \{ 0,1\}$ also play an important role in causal effect estimation.
For technical simplicity, we assume that 
$\sup_{t\in \{0,1\}} \sup_{\bm{x} \in \mathcal{X}} \mathbb{V}(Y (t) |\bm{x}) < \infty$,  
$\sup_{t\in \{0,1\}} \sup_{\bm{x} \in \mathcal{X}} m_t (\bm{x}) \leq B_m$ 
for a constant $B_m>0$,
and $\mathcal{X} \subset \mathbb{R}^d$ is compact.

For an integer $n \in \mathbb{N}$, we denote $[n] := \{1,\dots,n\}$. 
A capital letter denotes a random variable or matrix interchangeably
whenever its meaning is clear, and a vector is denoted by a bold letter, e.g. $\bm{x} := (x_1 , \dots, x_d)^{\top}.$
For $\epsilon>0$ and a set of functions $\mathcal{F} \subset L_1 (\mathbbm{P})$,
$\mathcal{N}_{[ \,]}(\mathcal{F}, L_1 (\mathbbm{P}), \epsilon)$ denotes $L_1 (\mathbbm{P})$-bracketing number of $\mathcal{F}$ \cite{gine2021mathematical}.

\subsection{Review of weighting methods}

In this paper, we consider the estimator of the ATT given as the following form
\begin{equation}
\widehat{\operatorname{ATT}}^{\bm{w}}=\sum_{i : T_i = 1}   \frac{1}{n_1}Y_i- \sum_{i : T_i = 0} w_i Y_i,
\label{ATT_estimator}
\end{equation}
for a given weight vector $\bm{w}=(w_1,\ldots,w_n)^\top$ with $w_i \ge 0,$ which we call the {\it weighted estimator}.
We review the methods of estimating $\bm{w}$.

\paragraph{Inverse Probability Weighting (IPW)} 
The IPW estimator for the ATT is given by 
\begin{equation*}
   \sum_{i : T_i = 1}   \frac{1}{n_1}Y_i- \sum_{i : T_i = 0} \frac{1}{n_1}\frac{\hat{\pi}\left(\bm{X}_i\right)}{\left(1-\hat{\pi}\left(\bm{X}_i\right)\right)} Y_i,
\end{equation*}
where $\hat{\pi}(\cdot)$ is an estimated propensity score. The quantities $(1/n_1) \{\hat{\pi}\left(X_i\right)/(1- \hat{\pi}\left(X_i\right))\}$ are the weights for control units. The propensity score is generally estimated by the maximum likelihood estimator (MLE) with
the linear logistic regression model: 
\begin{align*}
    \pi_{\bm{\beta}} (\bm{x}) = \frac{1}{1 + \exp(- \bm{x}^{\top} \bm{\beta})}.
\end{align*}
where $\bm{\beta} \in \mathbb{R}^d$. Other machine learning techniques can be also used instead \cite{lee2010improving}. 

A modified version of the IPW estimator is to use normalized weights such as 
\begin{equation*}
    \sum_{i : T_i = 1} \frac{1}{n_1} Y_i- \frac{\sum_{i : T_i = 0} \hat{\pi}\left(\bm{X}_i\right)\left(1-\hat{\pi}\left(\bm{X}_i\right)\right)^{-1} Y_i}{\sum_{i : T_i = 0} \hat{\pi}\left(\bm{X}_i\right)\left(1-\hat{\pi}\left(\bm{X}_i\right)\right)^{-1}}, \label{SIPW}
\end{equation*}
This estimator is often called the stabilized IPW (SIPW) \cite{robins2000marginal}.
The advantages of the SIPW estimator over the IPW are that
the SIPW is translation invariant in the sense that
the estimator is not affected by the way the outcomes are centered and
it is bounded by $(\min Y_i, \max Y_i).$

\paragraph{Covariate Balancing Propensity Score (CBPS)}
For an arbitrary measurable function $\bm \phi:\mathcal{X} \to \mathbb{R}^p$, we have
\begin{align*}
  \mathbb{E}\left[\left.\frac{\pi(\bm X)(1-T)\phi(\bm X)}{1-\pi(\bm X)}\right|\bm X\right] &=
  \frac{\phi(\bm X) \pi(\bm X)}{1-\pi(\bm X)}\mathbb{E}\left[(1-T)|\bm X\right] \\ &=
  \phi(\bm X) \mathbb{E}\left[T|\bm X\right] \\ &=
  \mathbb{E}\left[T \phi(\bm X) |\bm X\right]
\end{align*}
under the ignorability assumption. Hence
\begin{align}
  \mathbb{E}\left[\frac{\pi(\bm X)}{1-\pi(\bm X)} (1-T)\bm \phi(\bm X)\right] =
  \mathbb{E}\left[T\bm \phi(\bm X)\right].    \label{eq_unique}         
\end{align}
Note that the true propensity odds ratio, namely $\pi(\bm X)/(1-\pi(\bm X))$, is the unique measurable function (with respect to $\bm{X}$) that achieve the equality (\ref{eq_unique}).
            
Based on this intuition,
\citet{imai2014covariate} proposes to estimate
$\bm{\beta} \in \mathbb{R}^d$ 
by balancing the moments of the treated and control groups, 
\begin{align}
\frac{1}{n} \sum_{i : T_i = 0} \frac{\pi_{\bm{\beta}} (\bm{X}_i)}{1-\pi_{\bm{\beta}} (\bm{X}_i) } \bm{\phi} \left(\bm{X}_i\right) = \frac{1}{n} \sum_{i : T_i = 1} \bm{\phi}\left(\bm{X}_i\right),
\label{eq:CBPS-eq}
\end{align}
where $\bm{\phi}\left(\cdot\right) : \mathcal{X} \to \mathbb{R}^p$ is pre-specified transformation.
Then, they estimate the ATT by
 \begin{equation*}
     \sum_{i : T_i = 1} \frac{1}{n_1} Y_i- \sum_{i : T_i = 0} \frac{1}{n_1}\frac{\pi_{\hat{\bm{\beta}}} (\bm{X}_i)}{1-\pi_{\hat{\bm{\beta}}} (\bm{X}_i)} Y_i,
 \end{equation*}
where $\hat{\bm{\beta}}$ is the solution of the equation (\ref{eq:CBPS-eq}).
Especially, letting $\bm{\phi}\left(\bm{X}\right) = \bm{X}$ ensures that the first moment of each covariate is balanced even when the propensity score model is misspecified.
Thus, the corresponding ATT estimator is unbiased as long as the true outcome regression model is linear (see Appendix \ref{CBPS_unbiased} for the proof). 
Detailed procedures of CBPS and the extensions can be found in \citet{imai2014covariate}.


\paragraph{Entropy Balancing (EB)}
\citet{hainmueller2012entropy} proposes to maximize the entropy of the weights $\bm{w}$ 
while matching the moments of the two groups. They solve 
\begin{align*}
\underset{\bm{w}}{\operatorname{minimize}} & \sum_{i : T_i = 0} w_i \log w_i\\
\text { subject to } & \sum_{i : T_i = 0} w_i \bm{\phi}\left(\bm{X}_i\right)=\frac{1}{n_1} \sum_{i : T_i = 1} \bm{\phi}\left(\bm{X}_i\right),\\
&\sum_{i : T_i = 0} w_i=1, \quad
w_i>0,
\end{align*}
where $\bm{w} = (w_1 , \dots, w_n)^{\top}$ and $\bm{\phi}\left(\cdot\right) : \mathcal{X} \to \mathbb{R}^p$ is a  pre-specified transformation.
Then, they estimate the ATT by
\begin{equation*}
    \sum_{i : T_i = 1} \frac{Y_i}{n_1}-\sum_{i : T_i = 0} w_i Y_i. 
\end{equation*}
Note that $w_i$ play a similar role to $\frac{\pi(\bm{X}_i)}{1-\pi(\bm{X}_i)}$ in (\ref{eq_unique}) for CBPS and $\frac{\hat{\pi}\left(\bm{X}_i\right)}{\left(1-\hat{\pi}\left(\bm{X}_i\right)\right)}$ of the IPW estimator. 
That is, instead of estimating the propensity score, EB estimates the weight that satisfies the 
balancing condition like (\ref{eq:CBPS-eq}). When there are multiple solutions satisfying the balancing condition, EB chooses one which minimizes the entropy.
            
Later, \citet{zhao2017entropy} proves that EB is doubly robust in the sense that the estimator is consistent if the true outcome regression model or the logit of the propensity score is a linear function of $\bm{\phi}\left(\bm{X}\right)$. The estimator, however, is not consistent when neither of these two models is linear. 

\section{Bias and IPM for the weighted estimator}

In this section, we link the bias of the weighted estimator of the ATT 
to the IPM between the two weighted empirical distributions.  
We define 
\begin{align*}
    \mathcal{W}^{+} := \Bigg\{ & \bm{w} = (w_1 , \dots, w_n)^{\top} \in [0,1]^n :\\
& \sum_{i : T_i = 0} w_i  = 1 , \sum_{i : T_i = 1} w_i  = 0\Bigg\},
\end{align*}
and we only consider the weighted estimator with $\bm{w}\in \mathcal{W}^{+}.$

\subsection{Balancing error of the weighted estimator}


As discussed in \citet{ben2021balancing}, the error of $\widehat{\operatorname{ATT}}^{\bm{w}}$ can be decomposed as
\begin{align}
    \widehat{\operatorname{ATT}}^{\bm{w}} - \operatorname{ATT} 
    =  \operatorname{err}_{\text{bal}}^{\bm{w}} + \operatorname{err}_{\text{obs}}^{\bm{w}} + (\operatorname{SATT} - \operatorname{ATT}), \label{decom}
\end{align}
where $\operatorname{err}_{\text{bal}}^{\bm{w}}$ and $\operatorname{err}_{\text{obs}}^{\bm{w}}$ are the balancing and observation errors, respectively which are defined as
\begin{align}
    \operatorname{err}_{\text{bal}}^{\bm{w}} =&  \sum_{i : T_i = 1} \frac{m_0 (\bm{X}_i)}{n_1} - \sum_{i : T_i = 0} w_i m_0 (\bm{X}_i) , \nonumber \\
    \operatorname{err}_{\text{obs}}^{\bm{w}} =& \sum_{i : T_i = 1} \frac{Y_i - m_1(\bm{X}_i)}{n_1} 
    -  \sum_{i : T_i = 0} w_i (Y_i - m_0(\bm{X}_i)). \nonumber
\end{align}
See Appendix \ref{decom_proof} for the derivation of (\ref{decom}).
The observation error is an inevitable error 
due to the randomness in $Y$, and is unbiased in the sense that $\mathbbm{E} (\operatorname{err}_{\text{obs}}^{\bm{w}}) = 0$ for any $\bm{w} \in \mathcal{W}^{+}.$
Moreover, it can be shown that $\operatorname{err}_{\text{obs}}^{\bm{w}} \to 0$ holds 
as $n\to \infty$ under mild regularity conditions (e.g. $\sum_{i=1}^n w_i^2 \to 0$). See (\ref{obszero}) of the Appendix for the proof.

Hence, we focus on finding the weights that minimize the balancing error. The balancing error arises due to the covariate imbalance between the treated and weighted control units. Note that 
the balancing error is independent of the randomness of $Y$. Thus, if $\bm{w}$ balances the two covariate distributions perfectly, i.e.,
$$\sum_{i : T_i = 1} \frac{1}{n} \delta_{\bm{X}_i}(\cdot) = \sum_{i : T_i = 0} w_i \delta_{\bm{X}_i}(\cdot),$$
where $\delta_{\bm{x}}(\cdot)$ is the Dirac delta,
then the balancing error becomes zero regardless of what the $m_0$ is.

Note that CBPS and EB target to balance the first moments of $\phi(\bm{X})$
(i.e., using $\bm{w}$ which satisfies $\sum_{i : T_i = 1} \frac{1}{n_1} \bm{\phi}(\bm{X})_i =  \sum_{i : T_i = 0} w_i \bm{\phi}(\bm{X})_i$). Thus, balancing them guarantees the balancing error being 
zero only when $m_0(\cdot)$ is a linear combination of $\bm{\phi}(\cdot).$
The knowledge about $\bm{\phi}$, however, is rarely available in practice.

\subsection{The IPM as the worst-case balancing error}
Let
    \begin{align*}
        \mathbb{P}_{0,n}(\cdot) &= \frac{1}{n_0} \sum_{i : T_i = 0} \delta_{\bm{X}_i}(\cdot), \\
        \mathbb{P}_{1,n}(\cdot) &= \frac{1}{n_1} \sum_{i : T_i = 1} \delta_{\bm{X}_i}(\cdot)
    \end{align*}
be the empirical distributions of $\bm{X}$ conditioned on $T=0$ and $T=1$, respectively.
For $\bm{w} = (w_1 , \dots, w_n)^{\top} \in \mathcal{W}^{+}$, the weighted empirical distribution
$\mathbb{P}_{0,n}^{\bm{w}}$ of $\bm{X}$ in the control units  is defined as
\begin{align*}
    \mathbb{P}_{0,n}^{\bm{w}}(\cdot) &= \sum_{i : T_i = 0} w_i \delta_{\bm{X}_i}(\cdot).
\end{align*}
Although detailed procedures are different, the ultimate goal of the weighting methods 
is to find a good $\bm{w}$ such that $\mathbb{P}_{0,n}^{\bm{w}} \approx \mathbb{P}_{1,n}$.

The main idea of our proposed methods is to use the IPM  between $\mathbb{P}_{0,n}^{\bm{w}}$ and $\mathbb{P}_{1,n}$ for a measure of covariate imbalance.
For a given class $\mathcal{M}$ of discriminators (i.e. functions from $\mathcal{X}$ to $\mathbb{R}$), the IPM $d_{\mathcal{M}}(\mathbb{P}_1, \mathbb{P}_2)$ between two probability measures $\mathbb{P}_1$ and $\mathbb{P}_2$ is defined as
$$d_{\mathcal{M}}(\mathbb{P}_1, \mathbb{P}_2 ) := \sup_{m\in \mathcal{M}} \left| \int_{\bm{x} \in \mathcal{X}} m(\bm{x}) (d\mathbb{P}_1-d\mathbb{P}_2)\right|.$$
When $\mathcal{M}$ includes all 1-Lipschitz functions\footnote{A given function $m$ on $\mathcal{X}$ is a $L$-Lipschitz function if
$|m(\bm{x}_1)-m(\bm{x}_2)| \le L \|\bm{x}_1-\bm{x}_2\|$ for all $\bm{x}_1,\bm{x}_2\in \mathcal{X},$
 where $\|\cdot\|$ is certain norm defined on $\mathcal{X}.$}, the IPM becomes the well known Wasserstein distance \cite{kantorovich1958space}.

Note that the IPM between $\mathbbm{P}_{0,n}^{\bm{w}}$ and $\mathbbm{P}_{1,n}$ is given as
\begingroup\makeatletter\def\f@size{9.1}\check@mathfonts
\def\maketag@@@#1{\hbox{\m@th\normalfont#1}}%
\begin{align}
d_{\mathcal{M}}(\mathbbm{P}_{0,n}^{\bm{w}}, \mathbbm{P}_{1,n} ) =& \sup_{m\in \mathcal{M}} \left|  \sum_{i : T_i = 0} w_i m (\bm{X}_i) - \sum_{i : T_i = 1} \frac{m (\bm{X}_i)}{n_1} \right|. \label{errbal_IPM}
\end{align}\endgroup
That is, $d_{\mathcal{M}}(\mathbbm{P}_{0,n}^{\bm{w}}, \mathbbm{P}_{1,n} )$ is equal to the absolute value of the worst-case balancing error for the ATT when $m_0 \in \mathcal{M}$. 
Furthermore, since $\mathbb{E}(\operatorname{err}_{\text{obs}}^{\bm{w}})=0$, 
(\ref{decom}) implies that
the bias of $\widehat{\operatorname{ATT}}^{\bm{w}}$ is upper bounded by $d_{\mathcal{M}}(\mathbbm{P}_{0,n}^{\bm{w}}, \mathbbm{P}_{1,n} )$.
This property of the IPM is summarized in the following proposition. 
\begin{proposition} \label{proposition}
Suppose $m_0 \in \mathcal{M}.$ Then, for any $\bm{w}\in \mathcal{W}^{+}$, we have
\begin{align*}
    \left|\mathbb{E}\left(\widehat{\operatorname{ATT}}^{\bm{w}} - \operatorname{SATT} | \bm{X}_1 , \dots, \bm{X}_n\right)\right| \leq & d_{\mathcal{M}}(\mathbbm{P}_{0,n}^{\bm{w}}, \mathbbm{P}_{1,n} ).
\end{align*} 
\end{proposition}

\section{Covariate balancing using the IPM} \label{sec4}
In this section, we propose two methods for covariate balancing using the IPM (CBIPM).
The basic idea of the CBIPM is to estimate $\bm{w}$ by
\begin{align}
    \widehat{\bm{w}} = \underset{\bm{w} \in \mathcal{W}}{\operatorname{argmin}} \  d_{\mathcal{M}}(\mathbbm{P}_{0,n}^{\bm{w}}, \mathbbm{P}_{1,n} ), \label{solve_w}
\end{align}
where $\mathcal{W} \subseteq \mathcal{W}^{+}$ is the pre-specified set of weight vectors
and $\mathcal{M}$ is the set of discriminators.
We consider the two CBIPM - parametric CBIPM and nonparametric CBIPM which differ in
the choice of $\mathcal{W}$ and $\mathcal{M}.$

\subsection{Parametric CBIPM} \label{sec4_1}

If we have information that the propensity score belongs to some specified parametric family, we can use parametric space for $\mathcal{W}$.
Assume 
$$\operatorname{logit}(\pi(\cdot)) = c_0 + f( \ \cdot \ ; \bm{\theta}_0)$$
holds for unknown $c_0 \in \mathbb{R}$ and $\bm{\theta}_0 \in \Theta,$ where
$\Theta$ is a compact set of  $\mathbb{R}^k$ for  $k \in \mathbb{N}$ and
$f(\ \cdot \ ; \bm{\theta})$ is a function parameterized by $\bm{\theta}  \in \Theta.$ 
For the identifiability of the parameters, we assume
$f(\bm{0} ; \bm{\theta})=0$ for every $\bm{\theta} \in \Theta$.  
In this case, we consider
\begin{align*}
     \mathcal{W}^{P}(f) := \Big\{\bm{w}_{f}(\bm{\theta} \ ; \mathcal{D}^{(n)}): \bm{\theta} \in \Theta \Big\},
\end{align*}
where $\bm{w}_{f}(\ \cdot \ ; \mathcal{D}^{(n)}) : \Theta \to \mathcal{W}^{+}$ is an n-dimensional vector function defined as
\begin{align*}
     \bm{w}_{f}(\bm{\theta} ; \mathcal{D}^{(n)})_{i} := \frac{\mathbb{I}(T_i = 0) \exp(f(\bm{X}_i ; \bm{\theta}))}{ \sum_{i : T_i = 0} \exp({f(\bm{X}_i ; \bm{\theta}))}}, && i \in [n].
\end{align*}
Note that $\bm{w}_f(\bm{\theta}_0 ; \mathcal{D}^{(n)})$ is equivalent to the SIPW weights with the true propensity score. In other words, $\mathcal{W}(f)$ includes the ideal weight.
Finally, the parametric CBIPM method (P-CBIPM) solves
\begin{align}
    \widehat{\bm{w}} = \underset{\bm{w} \in \mathcal{W}^{P}(f)}{\operatorname{argmin}} \  d_{\mathcal{M}}(\mathbbm{P}_{0,n}^{\bm{w}}, \mathbbm{P}_{1,n} ), \label{solve_P_CBIPM}
\end{align}
and estimates the ATT using (\ref{ATT_estimator}).

The consistency of the ATT estimator of the parametric CBIPM is proved 
in Theorem \ref{theorem1}
under the following
very mild regularity conditions when the parametric model is correctly specified.

\begin{assumpV} \label{assum_fconti}
For every $\bm{x} \in \mathcal{X}$, $\bm{\theta} \mapsto f(\bm{x} ; \bm{\theta})$ is continuous. 
\end{assumpV}
 \begin{assumpV} 
 \label{assum_fiden}
 $f( \ \cdot \ ;\bm{\theta}_1) \equiv f(\ \cdot \ ; \bm{\theta}_2)$ if and only if
 $\bm{\theta}_1 = \bm{\theta}_2$
\end{assumpV}

\begin{assumpM} \label{assum_mfinite}
For any $\epsilon > 0$, $\mathcal{N}_{[ \,]}(\mathcal{M}, || \cdot ||_{1,\mathbb{P}}, \epsilon) < \infty$.
\end{assumpM}
\begin{assumpM} \label{assum_mzero}
$d_{\mathcal{M}}(\mathbb{P}_1, \mathbb{P}_2 ) = 0$ if and only if $\mathbb{P}_1 \equiv \mathbb{P}_2$.
\end{assumpM}

Assumptions \ref{assum_fconti} and \ref{assum_fiden} are 
about $f( \ \cdot \ ;\bm{\theta}),$ while  
Assumptions \ref{assum_mfinite} and \ref{assum_mzero} are about $\mathcal{M}.$
 Assumption \ref{assum_fconti} is made for technical purposes and
Assumption \ref{assum_fiden} is needed for the identifiability, both of which
are minimal.
While Assumption \ref{assum_mfinite} is a standard assumption for consistency,
Assumption \ref{assum_mzero} implies that the set of discriminators can be quite small.
For example, $\mathcal{M}=\{\exp(\bm{\theta}^\top \bm{x}): \bm{\theta} \in [-\tau,\tau]^{d}\}$
for some $\tau>0$ satisfies \ref{assum_mzero} because of the uniqueness of the moment generating function. 

\begin{theorem} \label{theorem1}
Assume there exist unknown $c_0 \in \mathbb{R}$ and $\bm{\theta}_0 \in \Theta$ such that
$$ \operatorname{logit}(\pi(\cdot)) = c_0 + f_{\bm{\theta}_0} (\cdot).$$
If Assumptions \ref{assum_fconti}, \ref{assum_fiden}, \ref{assum_mfinite} and \ref{assum_mzero} hold, then for $\widehat{\bm{w}}$ defined by (\ref{solve_P_CBIPM}),
\begin{align*}
    \widehat{\operatorname{ATT}}^{\widehat{\bm{w}}} \to \operatorname{ATT}.
\end{align*}
in probability.
\end{theorem}

\paragraph{Comparison with CBPS}
Note that the weight parameterizations of the P-CBIPM and CBPS are identical since $f(\ \cdot \ ; \bm{\theta})$ in the P-CBIPM  plays exactly the same role as $\operatorname{logit}(\pi_{\hat{\bm{\theta}}}(\cdot))$ of CBPS.
However, CBPS focuses matching the first moments, while the quantities to be balanced by the P-CBIPM
depend on $\mathcal{M}$. 
If we choose $\mathcal{M}$ as the set of linear functions, then the P-CBIPM is identical to CBPS. 
That is, CBPS can be considered as a special case of the P-CBIPM.
More details about equivalence are provided in Appendix \ref{CBPS_equ}.

\subsection{Nonparametric CBIPM} \label{sec4_2}
For the nonparametric CBIPM (N-CBIPM), we consider
\begin{align*}
     \mathcal{W}^N (B) := \Big\{\bm{w} \in \mathcal{W}^{+} : \max_{i \in [n]} w_i \leq \frac{B}{n_0} \Big\},
\end{align*}
as $\mathcal{W},$
where  $B > 0$ is a sufficiently large number such that $1/\eta^2 \leq B.$
Then, we solve
\begin{align}
    \widehat{\bm{w}} = \underset{\bm{w} \in \mathcal{W}^N (B)}{\operatorname{argmin}} \  d_{\mathcal{M}}(\mathbbm{P}_{0,n}^{\bm{w}}, \mathbbm{P}_{1,n} ), \label{solve_N_CBIPM}
\end{align}
and estimate the ATT by (\ref{ATT_estimator}).

As expected, we need a stronger assumption about $\mathcal{M}$ than Assumption \ref{assum_mzero} for the consistency of the ATT estimator, which is stated as follows.
\begin{definition}
For two classes $\mathcal{M}_1$ and $\mathcal{M}_2$ of discriminators, $d_{\mathcal{M}_1}(\cdot, \cdot) \lesssim d_{\mathcal{M}_2}(\cdot, \cdot)$ if and only if there exists an increasing function $\xi : [0,\infty) \to [0,\infty)$ such that $\lim_{r \downarrow 0} \xi(r) = 0$ and $d_{\mathcal{M}_1}(\mathbbm{P}_1, \mathbbm{P}_2) \leq \xi(d_{\mathcal{M}_2}(\mathbbm{P}_1, \mathbbm{P}_2))$ for any two probability measures $\mathbbm{P}_1$ and $\mathbbm{P}_2$.
\end{definition}

\begin{assumpM} ($\mathcal{M}$ dominates $\mathcal{M}_0$) \label{assum_msub}
    There exists a class $\mathcal{M}_0$ of outcome regression models including 
    $m_0(\cdot)$ such that
    $d_{\mathcal{M}_0}(\cdot, \cdot) \lesssim d_{\mathcal{M}}(\cdot, \cdot)$.
\end{assumpM}

Assumption \ref{assum_msub} provides us a guide for choosing 
$\mathcal{M}.$ Simply, we can choose $\mathcal{M} \supset \mathcal{M}_0$
because $d_{\mathcal{M}_0}(\cdot, \cdot) \lesssim d_{\mathcal{M}}(\cdot, \cdot)$
always holds. There exists, however, an interesting example of $\mathcal{M}$
such that $\mathcal{M}$ is fairly small (e.g. $\mathcal{M} \subset \mathcal{M}_0$)
but dominates a fairly large $\mathcal{M}_0$.
The sigmoid IPM \cite{pmlr-v162-kim22b} is such an example. See Section \ref{M_choice} for details.

\begin{theorem} \label{theorem2}
Consider $\mathcal{M}$ which satisfies Assumptions \ref{assum_mfinite} and \ref{assum_msub}.
Then for $\widehat{\bm{w}}$ defined by (\ref{solve_N_CBIPM}),
\begin{align*}
    \widehat{\operatorname{ATT}}^{\widehat{\bm{w}}} \to \operatorname{ATT}.
\end{align*}
in probability.
\end{theorem}

\subsection{Choice of the set of discriminators in the N-CBIPM} \label{M_choice}
\paragraph{Wassesterin distance}
Let $\mathcal{M}_{L_1}$ be the set of all bounded $1$-Lipschitz functions. i.e.,
$$\mathcal{M}_{L_1} := \{m(\cdot) : \mathcal{X} \to \mathbb{R}, ||m||_L \leq 1, ||m||_\infty \leq B_m \}.$$
Note that $d_{\mathcal{M}_{L_1}}(\cdot, \cdot)$ is the Wassesterin distance, which is widely used in machine learning society.
In practice, we 
approximate $1$-Lipschitz functions by deep neural networks (DNN) as it is done by \citet{arjovsky2017wasserstein} and \citet{gulrajani2017improved}. 
Details of the N-CBIPM method with the Wassesterin distance are given in Appendix \ref{app_C} for
reader's sake.


\paragraph{Maximum Mean Discrepancy}
Let $k_{\gamma}: \mathbb{R}^d \times \mathbb{R}^d \to \mathbb{R}$ be radial basis function (RBF) kernel with the width $\gamma$
and $(\mathcal{H}_{\gamma}(\mathcal{X}), ||\cdot||_{\mathcal{H}_{\gamma}(\mathcal{X})})$ be RKHS corresponding to $k_{\gamma}.$
For 
$$\mathcal{M}_{k_\gamma,B}=\{f\in \mathcal{H}_{\gamma}(\mathcal{X}): ||f||_{\mathcal{H}_{\gamma}(\mathcal{X})} \le B\},$$ 
Maximum Mean Discrepancy (MMD) is defined by $d_{\mathcal{M}_{k_\gamma,B}}(\cdot, \cdot)$  \cite{gretton2012kernel}.
The advantage of MMD is that we can calculate $d_{\mathcal{M}_{k_\gamma,B}}(\mathbbm{P}_{0,n}^{\bm{w}}, \mathbbm{P}_{1,n} )$ as a closed form and thus the corresponding optimization algorithm becomes 
lighter and simpler.

\paragraph{Sigmoid IPM} 

\citet{pmlr-v162-kim22b} studies the set of parametric discriminators, so-called
the sigmoid-IPM (SIPM), which is defined as
\begin{align*}
    \mathcal{M}_{sig} := \{ \sigma(\bm{\rho}^\top \bm{x}+\mu): \bm{\rho}\in \mathbb{R}^d,
\mu\in \mathbb{R}\},
\end{align*}
where $\sigma(\cdot)=(1+\exp(\cdot))^{-1}$ is the sigmoid function.
An interesting property of $\mathcal{M}_{sig}$ is that it dominates fairly large classes of discriminators even it is parametric. For example, $\mathcal{M}_{sig}$ dominates $\mathcal{M}_{\mathcal{C}^{\infty},K}$
with $\xi(r)= r$ and $\mathcal{F}_{a,C}$ with $\xi(r) = r^{1/3},$
where
\begin{align*}
	    \mathcal{M}_{\mathcal{C}^{\infty},K} 
	    := \big\{&f : \mathcal{X} \to \mathbb{R} \ : \ \forall \mathbf{r} \in \mathbb{N}_0^d,\\
	    & ||D^{\mathbf{r}} f||_{\infty} \leq \sqrt{{\mathbf{r}}!} K^{|\bm{r}|_1}
	    \big\},
\end{align*}
for positive constant $K$ and
$$\mathcal{F}_{a,C}=\left\{f: 
\int |\tilde{f}(\bm{z})|d\bm{z}\le a,
\int \|\bm{z}\|_1 |\tilde{f}(\bm{z})|d\bm{z}\le C\right\},$$
for positive constants $a$ and $C$, where $\tilde{f}(\bm{z})=\int e^{-i \bm{z}^\top \bm{x}} f(\bm{x}) d\bm{x}.$
Note that the function class $\mathcal{M}_{\mathcal{C}^{\infty},B}$ is large enough to include certain function spaces popularly used in modern machine learning algorithms, such as RKHS with RBF kernel. Thus, we expect that MMD behaves similarly to the SIPM.
On the other hand, $\mathcal{F}_{a,C}$ is the set of functions approximated by
single-layer neural networks \cite{yukich1995sup}, which is quite large.
For examples of function classes beloning to $\mathcal{F}_{a,C}$, see \citet{barron1993universal}.

\paragraph{Remark} 
These three $\mathcal{M}_{L_1}$, $\mathcal{M}_{k_\gamma,B}$ and $\mathcal{M}_{sig}$ satisfy Assumption \ref{assum_mfinite}.
See \citet{gottlieb2016adaptive}, \citet{sancetta2020estimation} and \citet{gao2007entropy} for the proofs.


\subsection{Computation algorithm}
We solve the min-max optimization problem (\ref{solve_w}) via the adversarial training algorithm.
Suppose that $\mathcal{W}$ and $\mathcal{M}$ are parameterized by $\bm{\theta}\in \Theta$
and $\bm{\psi}\in \Psi,$ respectively. 
For example, we can let 
$\bm{\theta}=(\theta_1,\ldots,\theta_n)^\top$ with  $w(\bm{\theta})_i \propto \exp(\theta_i) \mathbb{I}(T_i=0)$ for the N-CBIPM.
Then, we update the discriminator $m(\cdot ; \bm{\psi})$ (used in the IPM) 
and $\bm{w}(\bm{\theta})$ iteratively using the gradient ascent and descent algorithms. 
A pseudo-code for the CBIPM methods is described in Algorithm \ref{algATT}.
\begin{algorithm} 
	\caption{Proposed algorithm for the ATT} \label{algATT}
	\begin{algorithmic}[1]
        \FOR{$t = 1, \cdots, T$}
        \FOR{$t^{\prime} = 1, \cdots, T_{\text{adv}}$}
        \STATE Calculate $\mathcal{L}_{\textup{adv}}(\bm{\theta}, \bm{\psi})$ \label{alg_L1}
		\STATE $\bm{\psi} \gets \bm{\psi} + \textup{lr}_{\textup{adv}} \cdot\nabla_{\bm{\psi}} \mathcal{L}_{\textup{adv}}(\bm{\theta}, \bm{\psi})$ 
        \ENDFOR
        \STATE Calculate $\mathcal{L}(\bm{\theta}, \bm{\psi})$
		\STATE $\bm{\theta} \gets \bm{\theta} - \textup{lr} \cdot\nabla_{\bm{\theta}} \mathcal{L}(\bm{\theta}, \bm{\psi})$ \label{alg_L2}
        \ENDFOR
        \STATE \textbf{Return} $\sum_{i : T_i = 1} Y_i/n_1 - \sum_{i : T_i = 0} w(\bm{\theta})_{i} Y_i$
	\end{algorithmic}
\end{algorithm}

\begin{table*}[t!] 
\renewcommand{\arraystretch}{1.3}
\caption{\textbf{Kang-Schafer example with the linear/nonlinear propensity scores.} 
We generate 1000 simulations and report the bias and RMSE as the performance measures.
For each pair of dataset and performance measure (i.e. for each row), the two best results are marked by bold letters.} \label{table1}
\centering
\scalebox{1.0}{\small
\begin{tabular}{|c|c|c|cccc|ccc|ccc|}
\hline
\multirow{2}{*}{$\operatorname{logit}(\pi(\cdot))$} & \multirow{2}{*}{Measure} &  \multirow{2}{*}{n} & \multicolumn{4}{c|}{Existing methods} & \multicolumn{3}{c|}{P-CBIPM} & \multicolumn{3}{c|}{N-CBIPM} \\
   \cline{4-13}
& &   & GLM & Boost & CBPS & \multicolumn{1}{c|}{EB} & Wass & MMD & \multicolumn{1}{c|}{SIPM} & Wass & MMD & SIPM \\
    \hline \hline 
\multirow{4}{*}{Linear} & \multirow{2}{*}{Bias}  & 200 & 0.617 & 1.448 & -0.843 & -0.877 & -0.856 & -0.607 & \textbf{-0.423} & -1.040 & \textbf{0.346} & -0.576 \\
\cline{3-13}
& & 1000 & \textbf{-0.059} & 0.932 & -0.333 & -0.318 & -0.319 & -0.191 & -0.169 & -1.059  & 0.415 & \textbf{0.042}\\
\cline{2-13}
& \multirow{2}{*}{RMSE} & 200 & 4.026 & 3.935 & 3.018 & 3.020 & 3.029 & 2.932 & 2.842 & 2.876 & \textbf{2.615} & \textbf{2.715} \\
\cline{3-13}
& & 1000 & 2.838 & 1.586 & 1.433 & 1.433 & 1.457 & 1.395 & 1.299 & 1.563 & \textbf{1.097} & \textbf{1.147} \\
\hline \hline
\multirow{4}{*}{Nonlinear} & \multirow{2}{*}{Bias}  & 200 & -7.233  & -8.375 & -4.745 & -4.806 & -5.015 & -4.949 & -5.086 & -3.945 & \textbf{-3.162} & \textbf{-3.569} \\
\cline{3-13}
& & 1000 & -7.295 & -6.489 & -4.496 & -4.494  & -4.698 & -4.548 & -4.690 & -3.608  & \textbf{-2.844} & \textbf{-2.733}\\
\cline{2-13}
& \multirow{2}{*}{RMSE} & 200 & 8.275 & 9.204 & 5.354 & 5.395 & 5.681 & 5.562 & 5.697 & 4.632 & \textbf{4.356} & \textbf{4.491} \\
\cline{3-13}
& & 1000 & 7.514 & 6.665 & 4.620 & 4.618 & 4.887 & 4.700 & 4.826 & 3.761 & \textbf{3.020} & \textbf{2.973}\\
\hline
\end{tabular}
}
\end{table*}

The two loss functions $\mathcal{L}_{\textup{adv}}$ and $\mathcal{L}$ are modifications of
\begin{align*}
\ell_n(\bm{\theta}, \bm{\psi}) := & \left| \sum_{i : T_i = 1} \frac{m (\bm{X}_i ; \bm{\psi})}{n_1} - \sum_{i : T_i = 0} w(\bm{\theta})_i \, m (\bm{X}_i ; \bm{\psi}) \right|^2 ,
\end{align*}
where the modifications depend on the choice $\mathcal{M}.$
For example, there is a closed-form solution for updating $\bm{\psi}$ in MMD and thus only $\bm{\theta}$
is updated. For the SIPM, we propose to use an ensemble technique to avoid the phenomenon of mode collapse \cite{salimans2016improved, che2019mode}. The detailed explanations of the algorithms
for each IPM are given in Appendix \ref{app_C}.
We use the square of the IPM to make the loss function smooth.



\section{Experiments}

We illustrate the superiority of the CBIPM to existing baselines by analyzing simulated and real datasets. For the baselines, we consider the SIPW with the linear logistic regression (GLM), the SIPW with the boosting (Boost) used by \citet{lee2010improving},
the SIPW with the CBPS with the linear logistic regression  
\cite{imai2014covariate} and the  EB of \citet{hainmueller2012entropy}.
For CBPS and EB, we match the first moments of $\bm{X}$ (i.e. $\bm{\phi}(\bm{x})=\bm{x})$.
For the CBIPM, we use the three discriminators considered in Section \ref{M_choice}, and
the linear logistic regression is used for the P-CBIPM.

In Section \ref{sec5.1} and \ref{sec5.2}, we present the experimental results using simulation and real datasets respectively. Experimental results using other simulation settings are shown in Appendix \ref{App_another}. See Appendix \ref{App_semisyn} for the experiment results using the semi-synthetic datasets.
The code is available at \href{https://github.com/ggong369/CBIPM}{https://github.com/ggong369/CBIPM}.

\subsection{Simulation} \label{sec5.1}

We generate simulated datasets using the Kang-Schafer example \cite{kang2007demystifying}.
For each unit $i=1,\dots,n$, latent variables $\bm{Z}_i = (Z_{i1}, Z_{i2},Z_{i3},Z_{i4})^{\top}$ are independently generated from $N(0,I_4)$. 
Instead of $\bm{Z}_i$, only its nonlinear transformations $\bm{X}_i = (X_{i1}, X_{i2},X_{i3},X_{i4})^{\top}$ are observed, which are given as
\begin{align*}
\begin{split}
X_{i1} &= \exp(Z_{i1}/2), \\
X_{i2} &= Z_{i1} / (1+\exp(Z_{i1})) + 10, \\
X_{i3} &= (Z_{i1}Z_{i3}/25 + 0.6)^3, \\
X_{i4} &= (Z_{i2}+Z_{i4}+20)^2.
\end{split}
\end{align*}

Outcomes are generated from
\begin{align*}
Y_i = 210 + 27.4 Z_{i1} + 13.7 Z_{i2} +13.7 Z_{i3} + 13.7 Z_{i4} + \epsilon_i,
\end{align*}
where $\epsilon_i \sim N(0,1)$.
Note that the outcome regression models are nonlinear in $\bm{X}$
and that $\operatorname{ATT}=0$ since $m_0(\cdot) = m_1(\cdot).$  

For the true propensity score, we consider the linear and the nonlinear functions of $\bm{X}$.
Specifically, for the linear propensity score model, we generate the binary treatment indicators $T_i \in \{ 0,1 \}$ from
\begin{align*}
\mathbbm{P}(T=1|\bm{X}_i) = \sigma(X_{i1} - 0.5 X_{i2} -2 X_{i3} -0.01 X_{i4}).
\end{align*}
Also, for the nonlinear propensity score model, we use the model in the original Kang-Schafer example. That is, we generate the binary treatment indicators $T_i \in \{ 0,1 \}$ from 
\begin{align*}
\mathbbm{P}(T=1|\bm{Z}_i) = \sigma(-Z_{i1} + 0.5 Z_{i2} -0.25 Z_{i3} -0.1 Z_{i4}).
\end{align*}

The results based on 1000 simulated datasets are presented in
Table \ref{table1}. 
When $\operatorname{logit}(\pi(\cdot))$ is linear, GLM tends to have small biases, but it is unstable to have
large RMSEs.
In contrast, the parametric weighting methods such as
CBPS, EB, and the P-CBIPM have large biases but much smaller RMSEs compared to those of GLM. 

The results of the N-CBIPMs with the SIPM and MMD are much better than those of the parametric weighting methods,
which is surprising results since the parametric model is well specified. That is, 
the variance dominates the bias in the estimation of the inverse of the propensity score.
The superiority of the N-CBIPM is more prominent when $\operatorname{logit}(\pi(\cdot))$ is nonlinear.
For every $n$, the N-CBIPMs with the SIPM and MMD outperform the other methods with large margins in terms of
both the bias and RMSE.

\subsection{Real dataset} \label{sec5.2}
The Tennessee Student/Teacher Achievement Ratio experiment (STAR) is a 4-year longitudinal class-size study conducted by the State Department of Education to measure the influence of class size on student achievement tests and non-achievement measures \cite{STAR2008}. 
The school-level STAR dataset\footnote{https://doi.org/10.7910/DVN/SIWH9F} was collected in 1998 including the contents such as school demographic variables, school graduation rate, credits required for graduation, advanced course offerings, and so on. 
We analyze the STAR dataset to figure out 
how the covariance balancing is achieved by the weighting methods.


\begin{figure}[t]
\centering
\subfigure{\includegraphics[width=0.95\linewidth]{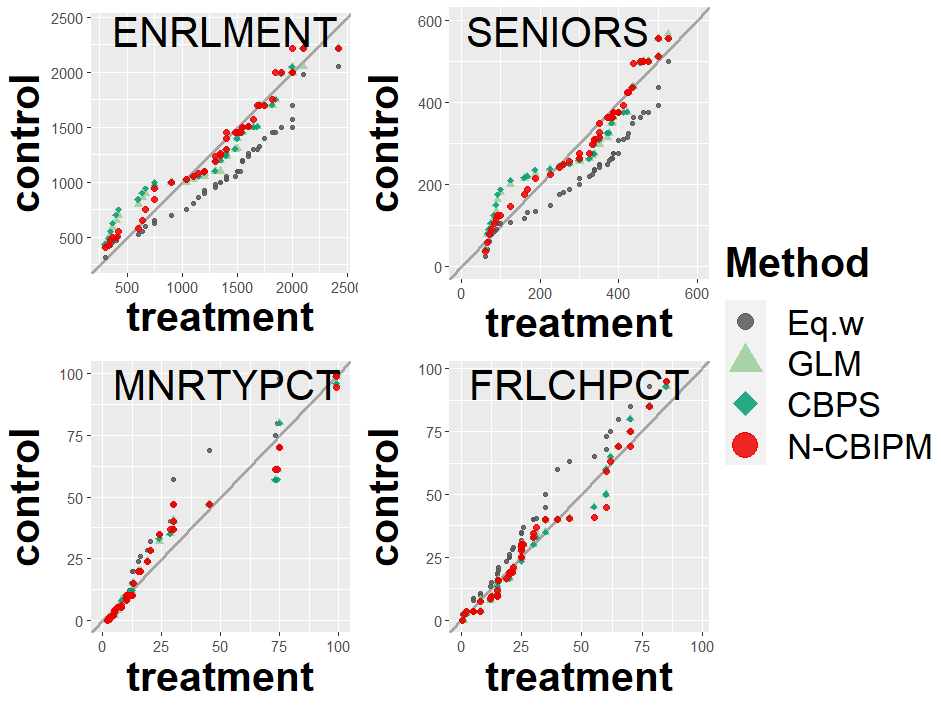}}
\caption{\textbf{Comparison of the QQ plots:} For the four input variables,
we draw the QQ plots of the four weighting methods between the weighted empirical  
distributions of the treated and the control groups.
} 
\vskip -0.1in \label{real_data}
\end{figure}

We use the six variables for $\bm{X}:$ school urbanicity (\textit{SCHLURBN}), student enrollment (\textit{ENRLMENT}), the estimated number of students in senior year (\textit{SENIORS}), percent of students minority (\textit{MNRTYPCT}), percent of students receiving free/reduced lunch (\textit{FRLCHPCT}), and percent of 9th-grade students in 94-95 who did not graduate (\textit{NOGRDPCT}),
and use whether the Math IV course is offered (\textit{MATH4}) or not for the treatment indicator.
To see how well the estimated weights achieve covariate balancing, 
we investigate the QQ plot between the marginal empirical distribution of the treated group
and the marginal weighted empirical distribution of the control group for the four continuous input variables
 (\textit{ENRLMENT, SENIORS, MNRTYPCT, FRLCHPCT}).
Figure \ref{real_data} compares the QQ plots for the four weighting methods including the equal weighting (Eq.w), GLM, CBPS, and the N-CBIPM with MMD (N-CBIPM). 
Note that the points of the QQ plot lines on the 45$^{\circ}$ straight line if and only if 
the two distributions are exactly the same \cite{marden2004positions}.
The four weighting methods improve covariate balancing much
compared to equal weighting, and the N-CBIPM performs better in particular for the \textit{SENIORS}.
See Appendix \ref{App_realdata} for the results of the hypothesis testing for the equality of the two distributions.

\section{Estimation of the ATE}
The CBIPM can be modified easily for estimating the ATE.
Let  $$\mathbb{P}_n(\cdot)=\frac{1}{n}\sum_{i=1}^n \delta_{\bm{X}_i}(\cdot)$$
be the empirical distribution of $\bm{X}.$
Then, we search for the weights of the control group as well as the weights for the treated group 
such that both of the two weighted empirical distributions are similar to  $\mathbb{P}_n.$
See Appendix \ref{app_ATE} for details.

\section{Discussions}

We have seen that the N-CBIPM dominates the other competitors. An important advantage
of the N-CBIPM is that the ATT (and the ATE) estimator is consistent even without
correctly specifying either the propensity score model or the outcome
regression model. Moreover, surprisingly,
the N-CBIPM outperforms the parametric weighting methods even when
the parametric model is correctly specified. For the ATE, it is well known that
the IPW estimator with an estimated propensity score is better than that with
the true estimated propensity score \cite{lunceford2004stratification}. A similar argument could be applied to explain
the superiority of the N-CBIPM.

The estimated weights by the N-CBIPM can be used for estimating the CATE or ITE. 
For example, the weighted empirical risk can be minimized to estimate
the outcome regression models in the T-learner \cite{kunzel2019metalearners}.
We believe that the weighted T-learner has several advantages over the T-learner, which
we leave as a future research topic. Similar modifications would be also possible for the S-, X- and R- learners \cite{kunzel2019metalearners, nie2021quasi}.

The convergence rate of the N-CBIPM estimator is also interesting. 
The choice of $\mathcal{M}$ and corresponding $\xi(\cdot)$
would affect the convergence rate. More studies on this problem are worth pursuing.

\section*{Acknowledgements}
This work was supported by National Research Foundation of Korea(NRF) grant funded by the Korea government (MSIT) (No. 2020R1A2C3A0100355014), 
and Convergence Research Center (CRC) grant funded by the Korea government (MSIT) (No. 2022R1A5A708390811).

\nocite{langley00}

\bibliography{references}
\bibliographystyle{icml2023}

\newpage
\appendix
\onecolumn

\numberwithin{equation}{section}
\section{Proofs for theoretical results}
\renewcommand{\theequation}{A.\arabic{equation}}

\paragraph{Proof for Proposition \ref{proposition}}
\begin{proof}
Since (\ref{errbal_IPM}),
\begin{align*}
    \left|\mathbb{E}\left(\widehat{\operatorname{ATT}}^{\bm{w}} - \operatorname{SATT} \Big| \bm{X}_1 , \dots, \bm{X}_n\right)\right|
    =& \left|\mathbb{E}( \operatorname{err}_{\text{bal}}^{\bm{w}} | \bm{X}_1 , \dots, \bm{X}_n)\right| \\
    \leq& d_{\mathcal{M}}(\mathbbm{P}_{0,n}^{\bm{w}}, \mathbbm{P}_{1,n} )
\end{align*} 
by (\ref{decom}).
\end{proof}

For simplicity, we denote $f_{\bm{\theta}} (\cdot) := f(\cdot ; \bm{\theta})$ for the proof.
\begin{lemma} \label{lemma1}
Let $\mathbbm{P}_0$ and $\mathbbm{P}_1$ be the probability measures of $\bm{X}$ conditioned on $T=0$ and $T=1$, respectively.
For given $\bm{\theta} \in \Theta$, denote
$$d\mathbbm{P}^{\bm{\theta}}_0 (\bm{x}) := 
\frac{\exp(f_{\bm{\theta}}(\bm{x})) d\mathbbm{P}_{0}(\bm{x})}
{\int_{\bm{x} \in \mathcal{X}} \exp(f_{\bm{\theta}}(\bm{x})) d\mathbbm{P}_{0}(\bm{x})}$$ 
as the weighted probability measure for the control group. 
If $\Theta \subset \mathbb{R}^k$ is compact and Assumption \ref{assum_fconti}, \ref{assum_fiden} and \ref{assum_mzero} hold, 
then for any $\epsilon>0$,
\begin{align}
    \underset{|\bm{\theta} - \bm{\theta}_0 | \geq \epsilon}{\inf} d_{\mathcal{M}} (\mathbbm{P}_{0}^{\bm{\theta}}, \mathbbm{P}_{1}) > 0. \label{lemma1eq}
\end{align}    
\end{lemma}
\begin{proof}
We denote $\mathbb{E}_0(\cdot)$ as the expectation operator with respect to $\mathbbm{P}_0$.
For any $\bm{x} \in \mathcal{X}$, since
$\exp(c_0 + f_{\bm{\theta}_0}(\bm{x})) = \pi(\bm{x})(1-\pi(\bm{x}))^{-1}$, we obtain
\begin{align}
    d\mathbbm{P}^{\bm{\theta}_0}_0 (\bm{x}) = & 
\frac{\exp(f_{\bm{\theta}_0}(\bm{x})) d\mathbbm{P}_{0}(\bm{x})}
{\int_{\bm{x} \in \mathcal{X}} \exp(f_{\bm{\theta}_0}(\bm{x})) d\mathbbm{P}_{0}(\bm{x})} \nonumber \\
=& \frac{\exp(c_0 + f_{\bm{\theta}_0}(\bm{x})) d\mathbbm{P}_{0}(\bm{x})}
{\int_{\bm{x} \in \mathcal{X}} \exp(c_0 + f_{\bm{\theta}_0}(\bm{x})) d\mathbbm{P}_{0}(\bm{x})} \nonumber \\
=& \frac{\pi(\bm{x})(1-\pi(\bm{x}))^{-1} d\mathbbm{P}_{0}(\bm{x})}
{\int_{\bm{x} \in \mathcal{X}} \pi(\bm{x})(1-\pi(\bm{x}))^{-1} d\mathbbm{P}_{0}(\bm{x})} \nonumber \\
=& \frac{d\mathbbm{P}_{1}(\bm{x}) \mathbbm{P}(T=1)/\mathbbm{P}(T=0)}
{(\int_{\bm{x} \in \mathcal{X}} d\mathbbm{P}_{1}(\bm{x})) \mathbbm{P}(T=1)/\mathbbm{P}(T=0)} \nonumber \\
=& d\mathbbm{P}_{1}(\bm{x}). \label{eqmeasure}
\end{align}

If $\mathbbm{P}_{0}^{\bm{\theta}} \equiv \mathbbm{P}_{1}$ holds for some $\bm{\theta} \in \Theta$, then
$$\frac{\exp(f_{\bm{\theta}}(\cdot)) }
{\mathbb{E}_0 (\exp(f_{\bm{\theta}}(\bm{X})))} \equiv \frac{\exp(f_{\bm{\theta}_0}(\cdot)) }
{\mathbb{E}_0 (\exp(f_{\bm{\theta}_0}(\bm{X})))},$$ 
and thus we get $f_{\bm{\theta}} \equiv f_{\bm{\theta}_0}$ since $f_{\bm{\theta}}(\bm{0}) = f_{\bm{\theta_0}}(\bm{0})=0$. 
Hence, by Assumption \ref{assum_fiden}, $\bm{\theta} = \bm{\theta}_0$.

To sum up, for any $\bm{\theta} \neq \bm{\theta}_0$, $d_{\mathcal{M}} (\mathbbm{P}_{0}^{\bm{\theta}}, \mathbbm{P}_{1}) > 0$ by Assumption \ref{assum_mzero}.
Since $\{\bm{\theta} \in \Theta : |\bm{\theta} - \bm{\theta}_0 | \geq \epsilon \}$ is compact and $ \bm{\theta} \mapsto d_{\mathcal{M}} (\mathbbm{P}_{0}^{\bm{\theta}}, \mathbbm{P}_{1})$ is continuous by Assumption \ref{assum_fconti}, we obtain (\ref{lemma1eq}).
\end{proof}

\begin{lemma}[Lemma 3.1 of \citet{van2000empirical}] \label{lemma2}
Let $\mathcal{F}$ be a set of functions from $\mathcal{X}$ to $\mathbb{R}$, and $\mathbb{P}_{n} = \frac{1}{n} \sum_{i = 1}^n \delta_{\bm{X}_i}$.
If for any $\epsilon > 0$, $\mathcal{N}_{[ \,]}(\mathcal{F}, L_1 (\mathbbm{P}), \epsilon) < \infty$, then
\begin{align*}
    \sup_{f \in \mathcal{F}} \left| \int_{\bm{x}} f(\bm{x}) (d\mathbbm{P}_n(\bm{x}) -  d\mathbbm{P}(\bm{x})) \right| \to 0.
\end{align*}    
\end{lemma}

\paragraph{Proof for Theorem \ref{theorem1}}
\begin{proof}
For given $\bm{\theta} \in \Theta$, denote
$$d\mathbbm{P}^{\bm{\theta}}_0 (\bm{x}) := 
\frac{\exp(f_{\bm{\theta}}(\bm{x})) d\mathbbm{P}_{0}(\bm{x})}
{\int_{\bm{x}} \exp(f_{\bm{\theta}}(\bm{x})) d\mathbbm{P}_{0}(\bm{x})}$$ 
as the weighted probability measure for the control group. 
Also, we deonote $\mathbbm{P}_{0,n}^{\bm{\theta}} := \mathbbm{P}_{0,n}^{\bm{w}_{f}(\bm{\theta} \ ; \mathcal{D}^{(n)})}$ and $\bm{w}^{\bm{\theta}} = (w^{\bm{\theta}}_1 , \dots, w^{\bm{\theta}}_n)^{\top} := \bm{w}_{f}(\bm{\theta} \ ; \mathcal{D}^{(n)})$.
For any given $\epsilon>0$, there exists $\delta>0$ by Lemma \ref{lemma1} such that 
\begin{align*}
    \underset{|\bm{\theta} - \bm{\theta}_0 | \geq \epsilon}{\inf} d_{\mathcal{M}} (\mathbbm{P}_{0}^{\bm{\theta}}, \mathbbm{P}_{1}) = \delta > 0.
\end{align*}
Note that solving (\ref{solve_P_CBIPM}) is identical to solving
$$
\widehat{\bm{\theta}}_n :=
\underset{\bm{\theta} \in \Theta}{\operatorname{argmin}} \
d_{\mathcal{M}}(\mathbbm{P}_{0,n}^{\bm{w}_f(\bm{\theta} \ ; \mathcal{D}^{(n)})}, \mathbbm{P}_{1,n} ).
$$
We will show that $\widehat{\bm{\theta}}_n$ converges to $\bm{\theta}_0$.
Since $d_{\mathcal{M}} (\mathbbm{P}_{0,n}^{\widehat{\bm{\theta}}_n}, \mathbbm{P}_{1,n}) \leq d_{\mathcal{M}} (\mathbbm{P}_{0,n}^{\bm{\theta}_0}, \mathbbm{P}_{1,n})$ by the definition of $\widehat{\bm{\theta}}_n$, we get
\begin{align}
    \mathbbm{P}_{\bm{\theta}_0}\left( |\widehat{\bm{\theta}}_n - \bm{\theta}_0 | \geq \epsilon \right)
    \leq& \mathbbm{P}_{\bm{\theta}_0}\left( \inf_{|\bm{\theta} - \bm{\theta}_0| \geq \epsilon} \left( d_{\mathcal{M}} (\mathbbm{P}_{0,n}^{\bm{\theta}}, \mathbbm{P}_{1,n}) - d_{\mathcal{M}} (\mathbbm{P}_{0,n}^{\bm{\theta}_0}, \mathbbm{P}_{1,n})\right)  \leq 0 \right) \nonumber \\
    \leq&  \mathbbm{P}_{\bm{\theta}_0}\left( \inf_{|\bm{\theta} - \bm{\theta}_0| \geq \epsilon} \left( d_{\mathcal{M}} (\mathbbm{P}_{0,n}^{\bm{\theta}}, \mathbbm{P}_{1,n}) - d_{\mathcal{M}} (\mathbbm{P}_{0,n}^{\bm{\theta}_0}, \mathbbm{P}_{1,n})\right)
    - \underset{|\bm{\theta} - \bm{\theta}_0 | \geq \epsilon}{\inf} d_{\mathcal{M}} (\mathbbm{P}_{0}^{\bm{\theta}}, \mathbbm{P}_{1}) < -\frac{\delta}{2} \right) \nonumber \\
    \leq & \mathbbm{P}_{\bm{\theta}_0}\left( \inf_{|\bm{\theta} - \bm{\theta}_0| \geq \epsilon} \left( d_{\mathcal{M}} (\mathbbm{P}_{0,n}^{\bm{\theta}}, \mathbbm{P}_{1,n}) - d_{\mathcal{M}} (\mathbbm{P}_{0,n}^{\bm{\theta}_0}, \mathbbm{P}_{1,n})
    - d_{\mathcal{M}} (\mathbbm{P}_{0}^{\bm{\theta}}, \mathbbm{P}_{1}) \right) < -\frac{\delta}{2} \right) \nonumber \\
    \leq&  \mathbbm{P}_{\bm{\theta}_0}\left( \sup_{\bm{\theta} \in \Theta} \left| d_{\mathcal{M}} (\mathbbm{P}_{0,n}^{\bm{\theta}}, \mathbbm{P}_{1,n}) - d_{\mathcal{M}} (\mathbbm{P}_{0}^{\bm{\theta}}, \mathbbm{P}_{1}) \right|  > \frac{\delta}{4} \right), \label{inprob}
\end{align}
where the last inequality holds because $d_{\mathcal{M}} (\mathbbm{P}_{0}^{\bm{\theta}_0}, \mathbbm{P}_{1}) = 0$ by (\ref{eqmeasure}).
Since
\begin{align*}
    & \left| d_{\mathcal{M}} (\mathbbm{P}_{0,n}^{\bm{\theta}}, \mathbbm{P}_{1,n}) - d_{\mathcal{M}} (\mathbbm{P}_{0}^{\bm{\theta}}, \mathbbm{P}_{1}) \right| \\
    =& \Bigg|\underset{m \in \mathcal{M}}{\sup} \Big| \sum_{i : T_i = 0} w_{i}^{\bm{\theta}} m(\bm{X}_i)  - \frac{1}{n_1} \sum_{i : T_i = 1} m(\bm{X}_i)  \Big|
    - \underset{m \in \mathcal{M}}{\sup} \Big| \int_{\bm{x} }  m(\bm{x}) d \mathbbm{P}_0^{\bm{\theta}}(\bm{x}) - \int_{\bm{x} }  m(\bm{x}) d\mathbbm{P}_1 (\bm{x}) \Big| \Bigg| \\
    \leq&  \underset{m \in \mathcal{M}}{\sup} \Big| \sum_{i : T_i = 0} w_{i}^{\bm{\theta}} m(\bm{X}_i) - \int_{\bm{x} }  m(\bm{x}) d\mathbbm{P}_0^{\bm{\theta}}(\bm{x}) \Big| 
    + \underset{m \in \mathcal{M}}{\sup} \Big| \frac{1}{n_1} \sum_{i : T_i = 1} m(\bm{X}_i)  - \int_{\bm{x} }  m(\bm{x}) d\mathbbm{P}_1 (\bm{x}) \Big|,
\end{align*}
we obtain
\begin{align}
\mathbbm{P}_{\bm{\theta}_0}\left( \sup_{\bm{\theta} \in \Theta} \left| d_{\mathcal{M}} (\mathbbm{P}_{0,n}^{\bm{\theta}}, \mathbbm{P}_{1,n}) - d_{\mathcal{M}} (\mathbbm{P}_{0}^{\bm{\theta}}, \mathbbm{P}_{1}) \right|  > \frac{\delta}{4} \right)
\leq & \mathbbm{P}_{\bm{\theta}_0}\left( \sup_{\bm{\theta} \in \Theta}  \underset{m \in \mathcal{M}}{\sup} \Big| \sum_{i : T_i = 0} w_{i}^{\bm{\theta}} m(\bm{X}_i) - \int_{\bm{x} }  m(\bm{x})d \mathbbm{P}_0^{\bm{\theta}}(\bm{x}) \Big| > \frac{\delta}{8}  \right) \label{ulln_1}\\
&+  \mathbbm{P}_{\bm{\theta}_0}\left( \underset{m \in \mathcal{M}}{\sup} \Big| \frac{1}{n_1} \sum_{i : T_i = 1} m(\bm{X}_i)  - \int_{\bm{x} }  m(\bm{x}) d\mathbbm{P}_1 (\bm{x}) \Big|  > \frac{\delta}{8}  \right) \label{ulln_2}.
\end{align}

Since 
$$\underset{\bm{\theta} \in \Theta, m \in \mathcal{M}}{\sup} \Bigg| \frac{1}{n_0} \sum_{i : T_i = 0} f_{\bm{\theta}}(\bm{X}_i) m(\bm{X}_i) - \int_{\bm{x} } f_{\bm{\theta}}(\bm{x}) m(\bm{x}) d \mathbbm{P}_0 (\bm{x}) \Bigg| \to 0, $$
and
$$\underset{\bm{\theta} \in \Theta}{\sup} \Bigg| \frac{1}{n_0} \sum_{i : T_i = 0} f_{\bm{\theta}}(\bm{X}_i)  - \int_{\bm{x} } f_{\bm{\theta}}(\bm{x}) d \mathbbm{P}_0 (\bm{x}) \Bigg| \to 0$$
hold by Assumption \ref{assum_mfinite} and Lemma \ref{lemma2}, we get
\begin{align*}
    \underset{\bm{\theta} \in \Theta, m \in \mathcal{M}}{\sup} \Bigg| \sum_{i : T_i = 0} w_{i}^{\bm{\theta}} m(\bm{X}_i) - \int_{\bm{x} }  m(\bm{x})dF_0^{\bm{\theta}}(\bm{x}) \Bigg|
    = & \underset{\bm{\theta} \in \Theta, m \in \mathcal{M}}{\sup} \Bigg| \frac{\sum_{i : T_i = 0} f_{\bm{\theta}}(\bm{X}_i) m(\bm{X}_i)}{ \sum_{i : T_i = 0} f_{\bm{\theta}}(\bm{X}_i)}
     - \frac{\int_{\bm{x} } f_{\bm{\theta}}(\bm{x}) m(\bm{x}) d\mathbbm{P}_0 (\bm{x})}{\int_{\bm{x} } f_{\bm{\theta}}(\bm{x}) d\mathbbm{P}_0 (\bm{x}) } \Bigg| \\
     \leq &  \underset{\bm{\theta} \in \Theta, m \in \mathcal{M}}{\sup}
     \Bigg| \frac{\frac{1}{n_0}\sum_{i : T_i = 0} f_{\bm{\theta}}(\bm{X}_i) m(\bm{X}_i)}{ \frac{1}{n_0} \sum_{i : T_i = 0} f_{\bm{\theta}}(\bm{X}_i)}
     - \frac{\int_{\bm{x} } f_{\bm{\theta}}(\bm{x}) m(\bm{x}) 
     d\mathbbm{P}_0 (\bm{x})}{\frac{1}{n_0} \sum_{i : T_i = 0} f_{\bm{\theta}}(\bm{X}_i)} \Bigg| \\
     & + \underset{\bm{\theta} \in \Theta, m \in \mathcal{M}}{\sup} \Bigg| \frac{\int_{\bm{x} } f_{\bm{\theta}}(\bm{x}) m(\bm{x}) d \mathbbm{P}_0 (\bm{x})}{ \frac{1}{n_0} \sum_{i : T_i = 0} f_{\bm{\theta}}(\bm{X}_i)}
     - \frac{\int_{\bm{x} } f_{\bm{\theta}}(\bm{x}) m(\bm{x}) d \mathbbm{P}_0 (\bm{x})}{\int_{\bm{x} } f_{\bm{\theta}}(\bm{x}) d \mathbbm{P}_0 (\bm{x}) } \Bigg| \\
     \to & 0,
\end{align*}
and hence (\ref{ulln_1}) converges to $0$.
Also, (\ref{ulln_2}) converges to $0$ by Lemma \ref{lemma2}.
To sum up, (\ref{inprob}) converges to $0$, which implies that $\widehat{\bm{\theta}}_n$ converges to $\bm{\theta}_0$ in probability.
Hence, by Assumption \ref{assum_fconti}, we have 
\begin{align*}
\widehat{\operatorname{ATT}}^{\widehat{\bm{w}}}
=& \sum_{i : T_i = 1}   \frac{Y_i}{n_1}- \frac{\sum_{i : T_i = 0} \exp(f(\bm{X}_i ; \widehat{\bm{\theta}}_n)) Y_i}{ \sum_{i : T_i = 0} \exp(f(\bm{X}_i ; \widehat{\bm{\theta}}_n))}\\
\to& \mathbb{E}\left(Y(1)|T=1\right) - \mathbb{E}\left(Y(0)|T=1\right) \\
=& \operatorname{ATT}.
\end{align*}
\end{proof}

\paragraph{Proof for Theorem \ref{theorem2}}
\begin{proof}
Let $u(\cdot) := \frac{\pi(\cdot)}{1-\pi(\cdot)}$. 
For given $\mathcal{D}_n$, we define
$\bm{w}^{0} = (w_{1}^{0}, \dots, w_{n}^{0})^{\top}$ as 
\begin{align*}
    w_{i}^{0} = \frac{\mathbb{I}(T_i = 0) u(\bm{X}_i)}{ \sum_{i : T_i = 0} {u(\bm{X}_i)}}, && i \in [n].
\end{align*}
Since
$$ \eta < \frac{\eta}{1-\eta} \leq u(\cdot) \leq \frac{1-\eta}{\eta}< \frac{1}{\eta},$$
$\bm{w}^{0} \in \mathcal{W}^N(B)$ holds. By definition of $\widehat{\bm{w}}$,
\begin{align}
    d_{\mathcal{M}}(\mathbbm{P}_{0,n}^{\widehat{\bm{w}}}, \mathbbm{P}_{1,n}) \leq  d_{\mathcal{M}}(\mathbbm{P}_{0,n}^{\bm{w}^{0}}, \mathbbm{P}_{1,n}). \label{smallIPM}
\end{align}

For any $m(\cdot)\in \mathcal{M}$, we have
\begin{align*}
\sum_{i : T_i = 1} \left( \frac{m(\bm{X}_i)}{n_1} \right) 
=& \frac{\sum_{i=1}^n \mathbb{I}(T=1) m(\bm{X}_i)}{\sum_{i=1}^n \mathbb{I}(T=1)} \\
= & \frac{\frac{1}{n}\sum_{i=1}^n \mathbb{I}(T=1) m(\bm{X}_i)}{\frac{1}{n}\sum_{i=1}^n \mathbb{I}(T=1)} \\
\to & \frac{\mathbb{E} (\mathbb{I}(T = 1) m(\bm{X}) ) }{\mathbb{E} (\mathbb{I}(T = 1)) } \\
= & \mathbb{E}(m(\bm{X})|T=1),
\end{align*}
and hence 
\begin{align}
\sup_{m\in \mathcal{M}} 
\left| \sum_{i : T_i = 1}\frac{m(\bm{X}_i)}{n_1} - 
\mathbb{E}(m(\bm{X})|T=1) \right| \to 0 \label{asconv1}
\end{align}
holds by Lemma \ref{lemma2}.

Also, since
\begin{align*}
\mathbb{E} (\mathbb{I}(T = 0) u (\bm{X}) m(\bm{X}) ) 
= &\mathbb{E} \left(\mathbb{E} \left(\mathbb{I}(T = 0) \frac{\pi(\bm{X})}{1-\pi(\bm{X})} m(\bm{X}) \Bigg| \bm{X}\right)\right) \\
= &\mathbb{E} \left( \frac{\pi(\bm{X})}{1-\pi(\bm{X})} m(\bm{X})  \mathbb{E} \left(\mathbb{I}(T = 0) | \bm{X}\right)\right) \\
= & \mathbb{E} (\mathbb{I}(T = 1) m(\bm{X})),
\end{align*}
and
\begin{align*}
\mathbb{E} (\mathbb{I}(T = 0) u (\bm{X})) 
= &\mathbb{E} \left(\mathbb{E} \left(\mathbb{I}(T = 0) \frac{\pi(\bm{X})}{1-\pi(\bm{X})} \Bigg| \bm{X}\right)\right) \\
= &\mathbb{E} \left( \frac{\pi(\bm{X})}{1-\pi(\bm{X})}  \mathbb{E} \left(\mathbb{I}(T = 0) | \bm{X}\right)\right) \\
= & \mathbb{E} (\mathbb{I}(T = 1)),
\end{align*}

we obtain
\begin{align*}
    \sum_{i : T_i = 0} w_{i}^{0} m(\bm{X}_i)
=& \frac{\sum_{i : T_i = 0} u (\bm{X}_i) m(\bm{X}_i)}{\sum_{i : T_i = 0} u(\bm{X}_i)}\\
=& \frac{\sum_{i=1}^n \mathbb{I}(T=0) u(\bm{X}_i) m(\bm{X}_i)}{\sum_{i=1}^n \mathbb{I}(T=0) u(\bm{X}_i)}\\
= & \frac{\frac{1}{n}\sum_{i=1}^n \mathbb{I}(T=0) u(\bm{X}_i) m(\bm{X}_i)}{\frac{1}{n}\sum_{i=1}^n \mathbb{I}(T=0) u(\bm{X}_i)} \\
\to & \frac{\mathbb{E} (\mathbb{I}(T = 0) u (\bm{X}) m(\bm{X}) ) }{\mathbb{E} (\mathbb{I}(T = 0) u (\bm{X})) } \\
= & \frac{\mathbb{E} (\mathbb{I}(T = 1) m(\bm{X}) ) }{\mathbb{E} (\mathbb{I}(T = 1)) } \\
= & \mathbb{E}(m(\bm{X})|T=1). 
\end{align*}
Hence, 
\begin{align}
    \sup_{m\in \mathcal{M}} 
\left| \sum_{i : T_i = 0} w_{i}^{0} m(\bm{X}_i) - 
\mathbb{E}(m(\bm{X})|T=1) \right| \to 0 \label{asconv2}
\end{align}
holds by Lemma \ref{lemma2}.

To sum up, 
\begin{align*}
    d_{\mathcal{M}}(\mathbbm{P}_{0,n}^{\widehat{\bm{w}}}, \mathbbm{P}_{1,n}) \leq & d_{\mathcal{M}}(\mathbbm{P}_{0,n}^{\bm{w}^{0}}, \mathbbm{P}_{1,n}) \\
    = & \sup_{m\in \mathcal{M}} 
 \left| \sum_{i : T_i = 0} w_{i}^{0} m(\bm{X}_i) - \sum_{i : T_i = 1}\frac{m(\bm{X}_i)}{n_1} \right| \\
\to &  0 
\end{align*}
holds by (\ref{smallIPM}), (\ref{asconv1}) and (\ref{asconv2}).
Since $d_{\mathcal{M}_0}(\cdot, \cdot) \lesssim d_{\mathcal{M}}(\cdot, \cdot)$ by Assumption \ref{assum_msub}, we obtain
\begin{align}
\operatorname{err}_{\text{bal}}^{\widehat{\bm{w}}} 
= & \sum_{i : T_i = 1} \frac{m_0 (\bm{X}_i)}{n_1} - \sum_{i : T_i = 0} \widehat{w}_{i} m_0 (\bm{X}_i) \nonumber\\
\leq & d_{\mathcal{M}_0}(\mathbbm{P}_{0,n}^{\widehat{\bm{w}}}, \mathbbm{P}_{1,n}) \nonumber \\
\leq & \xi(d_{\mathcal{M}}(\mathbbm{P}_{0,n}^{\widehat{\bm{w}}}, \mathbbm{P}_{1,n})) \nonumber \\
\to & 0. \label{balzero}
\end{align}

Let $\mathcal{C}^{(n)} = \{(\bm{X}_i, T_i)\}_{i=1}^n$.
Since $\widehat{w}_{i}$ is a measurable function of $\mathcal{C}^{(n)}$ for every $i \in [n]$, we get
$$\mathbb{E} \left(\sum_{i : T_i = 1} \frac{Y_i - m_1(\bm{X}_i)}{n_1} 
    -  \sum_{i : T_i = 0} \widehat{w}_{i} (Y_i - m_0(\bm{X}_i)) \Bigg| \mathcal{C}^{(n)}  \right)=0.$$
Also, since
\begin{align*}
\max_{i\in[n]} \widehat{w}_{i}
<  \frac{1}{n_0 \eta^2},
\end{align*}
we obtain
\begin{align*}
    \operatorname{Var} \left(\sum_{i : T_i = 1} \frac{Y_i - m_1(\bm{X}_i)}{n_1} 
    -  \sum_{i : T_i = 0} \widehat{w}_{n,i} (Y_i - m_0(\bm{X}_i)) \Bigg| \mathcal{C}^{(n)}  \right)
    = & \sum_{i : T_i = 1} \operatorname{Var} \left( \frac{Y_i(1)}{n_1} \Big| \bm{X}_i \right)
      + \sum_{i : T_i = 0} \operatorname{Var} \left( \widehat{w}_{i} Y_i(0) \big| \bm{X}_i \right)\\
    = & \sum_{i : T_i = 1} \frac{\operatorname{Var} \left( Y_i(1) | \bm{X}_i \right)}{n_1^2} 
      + \sum_{i : T_i = 0} \widehat{w}_{i}^2 \operatorname{Var} \left( Y_i(0) \big| \bm{X}_i \right)\\
    \leq & \frac{C}{min(n_0 , n_1) }.
\end{align*}
for some constant $C>0$.
In addition, we have
\begin{align}
\mathbb{E} \left(\sum_{i : T_i = 1} \frac{Y_i - m_1(\bm{X}_i)}{n_1} 
    -  \sum_{i : T_i = 0} \widehat{w}_{i}^2 (Y_i - m_0(\bm{X}_i)) \Big| \mathcal{C}^{(n)}  \right) 
    =0, \label{condiex_zero}
\end{align}
and so
\begin{align*}
    \operatorname{Var}(\operatorname{err}_{\text{obs}}^{\widehat{\bm{w}}}) 
    =& \operatorname{Var} \left(\sum_{i : T_i = 1} \frac{Y_i - m_1(\bm{X}_i)}{n_1} 
    -  \sum_{i : T_i = 0} \widehat{w}_{i} (Y_i - m_0(\bm{X}_i))\right)\\
    =& \mathbb{E} \left( \operatorname{Var} \left(\sum_{i : T_i = 1} \frac{Y_i - m_1(\bm{X}_i)}{n_1} 
    -  \sum_{i : T_i = 0} \widehat{w}_{i} (Y_i - m_0(\bm{X}_i)) \Big| \mathcal{C}^{(n)}  \right) \right)\\
    +&  \operatorname{Var} \left( \mathbb{E} \left(\sum_{i : T_i = 1} \frac{Y_i - m_1(\bm{X}_i)}{n_1} 
    -  \sum_{i : T_i = 0} \widehat{w}_{i} (Y_i - m_0(\bm{X}_i)) \Big| \mathcal{C}^{(n)} \right) \right)\\
    \to & 0,
\end{align*}
which implies 
\begin{align}     \operatorname{err}_{\text{obs}}^{\widehat{\bm{w}}} \to 0 \label{obszero}
\end{align}
in probability because $\mathbb{E} \left(\operatorname{err}_{\text{obs}}^{\widehat{\bm{w}}} \right) = 0$ by (\ref{condiex_zero}).
Now, by (\ref{balzero}), (\ref{obszero}) and the fact $\operatorname{SATT} \to \operatorname{ATT}$ in probability, we conclude $\widehat{\operatorname{ATT}}^{\widehat{\bm{w}}} \to \operatorname{ATT}$ in probability.
\end{proof}

\newpage
\section{The CBIPM for the ATE} \label{app_ATE}
\renewcommand{\theequation}{B.\arabic{equation}}
In this paper, we mainly focus on the ATT to write more concisely.
However, all discussions about the ATT can be extended to the ATE.
In this section, we briefly analyze the ATE using the CBIPM.

\subsection{Bias and the IPM for the ATE}
In this section, we link the bias of the weighted estimator
of the ATE to the IPM.

For $t \in \{ 0,1 \}$, we define 
\begin{align*}
    \mathcal{W}^{+}_t := \Bigg\{ & \bm{w} = (w_1 , \dots, w_n)^{\top} \in [0,1]^n :
 \sum_{i : T_i = t} w_i  = 1 , \sum_{i : T_i \neq t} w_i  = 0\Bigg\}
\end{align*}
as the set of all possible $n$-dimensional weight vectors for the units $\{i : T_i = t \}$.
The weighted estimator for the ATE using $\bm{w}_0 = (w_{0,1} , \dots, w_{0,n})^{\top} \in \mathcal{W}_0^{+}$ and $\bm{w}_1 = (w_{1,1} , \dots, w_{1,n})^{\top} \in \mathcal{W}_1^{+}$ can be expressed as
\begin{equation}
    \widehat{\operatorname{ATE}}^{\bm{w}_0, \bm{w}_1} = \sum_{i : T_i = 1} w_{1,i} Y_i -\sum_{i : T_i = 0} w_{0,i} Y_i. \label{ATE_estimator}
\end{equation}
The error of $\widehat{\operatorname{ATE}}^{\bm{w}_0, \bm{w}_1}$ can be decomposed as
\begin{align}
    \widehat{\operatorname{ATE}}^{\bm{w}_0, \bm{w}_1} - \operatorname{ATE} 
    =  \operatorname{Er}_{\text{bal}}^{\bm{w}_0, \bm{w}_1} + \operatorname{Er}_{\text{obs}}^{\bm{w}_0, \bm{w}_1} + (\operatorname{SATE} - \operatorname{ATE}), \label{decom_ATE}
\end{align}
where  
\begin{align*}
    \operatorname{Er}_{\text{bal}}^{\bm{w}_0, \bm{w}_1} =&  \left( \sum_{i= 1}^n \frac{m_0 (\bm{X}_i)}{n} - \sum_{i : T_i = 0} w_{0,i} m_0 (\bm{X}_i) \right) 
     - \left(\sum_{i= 1}^n \frac{m_1 (\bm{X}_i)}{n} - \sum_{i : T_i = 1} w_{1,i} m_1 (\bm{X}_i) \right),
\end{align*}
and
\begin{align*}
    \operatorname{Er}_{\text{obs}}^{\bm{w}_0, \bm{w}_1} =& 
    -  \sum_{i : T_i = 0} w_{0,i} (Y_i - m_0(\bm{X}_i)) + \sum_{i : T_i = 1} w_{1,i} (Y_i - m_1(\bm{X}_i)).
\end{align*}
Here, $\operatorname{Er}_{\text{bal}}^{\bm{w}_0, \bm{w}_1}$ and $\operatorname{Er}_{\text{obs}}^{\bm{w}_0, \bm{w}_1}$ are the balancing error observation errors of the ATE, which have similar properties with those of the ATT.
That is, $\operatorname{Er}_{\text{bal}}^{\bm{w}_0, \bm{w}_1}$ arises due to covariate
imbalance and $\operatorname{Er}_{\text{obs}}^{\bm{w}_0, \bm{w}_1}$ is an inevitable error due to the randomness in $Y$.

For $\bm{w}_t = (w_{t,1} , \dots, w_{t,n})^{\top} \in \mathcal{W}_t^{+}$ for $t \in \{ 0,1 \}$, we denote 
\begin{align*}
    \mathbb{P}_{t,n}^{\bm{w}_t} &= \sum_{i : T_i = t} w_{t,i} \delta_{\bm{X}_i},
\end{align*}
as the weighted empirical distribution of $\bm{X}$ for control (or treated) units. 
Since 
\begin{align}
d_{\mathcal{M}}(\mathbbm{P}_{t,n}^{\bm{w}_t}, \mathbbm{P}_{n}) =& \sup_{m(\cdot)\in \mathcal{M}} \left|  \sum_{i = 1}^n \frac{m (\bm{X}_i)}{n} - \sum_{i : T_i = t} w_{t,i} m (\bm{X}_i) \right|, \label{errbal_IPM_ATE}
\end{align}
$d_{\mathcal{M}}(\mathbbm{P}_{0,n}^{\bm{w}_0}, \mathbbm{P}_{n}) + d_{\mathcal{M}}(\mathbbm{P}_{1,n}^{\bm{w}_1}, \mathbbm{P}_{n})$ 
is an upper bound of the worst-case balancing error for the ATE when $m_0, m_1 \in \mathcal{M}$. 

Furthermore, since $\mathbb{E}(\operatorname{Er}_{\text{obs}}^{\bm{w}_0, \bm{w}_1})=0$, 
(\ref{decom_ATE}) implies that
the bias of $\widehat{\operatorname{ATE}}^{\bm{w}_0, \bm{w}_1}$ is upper bounded by $d_{\mathcal{M}}(\mathbbm{P}_{0,n}^{\bm{w}_0}, \mathbbm{P}_{n} ) + d_{\mathcal{M}}(\mathbbm{P}_{1,n}^{\bm{w}_1}, \mathbbm{P}_{n} )$.
\begin{proposition} \label{proposition_ATE}
If $m_0, m_1 \in \mathcal{M}$, then
\begin{align*}
    \left| \mathbb{E}(\widehat{\operatorname{ATE}}^{\bm{w}_0, \bm{w}_1} - \operatorname{SATE} | \bm{X}_1 , \dots, \bm{X}_n) \right| \leq d_{\mathcal{M}}(\mathbbm{P}_{0,n}^{\bm{w}_0}, \mathbbm{P}_{n} )
    + d_{\mathcal{M}}(\mathbbm{P}_{1,n}^{\bm{w}_1}, \mathbbm{P}_{n} )
\end{align*} 
holds.
\end{proposition}
\begin{proof}
Since (\ref{errbal_IPM_ATE}),
\begin{align*}
    \left|\mathbb{E}\left(\widehat{\operatorname{ATE}}^{\bm{w}_0, \bm{w}_1} - \operatorname{SATE} \Big| \bm{X}_1 , \dots, \bm{X}_n\right) \right|
    =& \left|\mathbb{E}( \operatorname{Er}_{\text{bal}}^{\bm{w}_0, \bm{w}_1} | \bm{X}_1 , \dots, \bm{X}_n)\right| \\
    \leq& d_{\mathcal{M}}(\mathbbm{P}_{0,n}^{\bm{w}_0}, \mathbbm{P}_{n} )
    + d_{\mathcal{M}}(\mathbbm{P}_{1,n}^{\bm{w}_1}, \mathbbm{P}_{n} )
\end{align*} 
by (\ref{decom_ATE}).
\end{proof}

\subsection{The CBIPM for the ATE}
The basic idea of the CBIPM for the ATE is to estimate $\bm{w}_0$ and $\bm{w}_1$ by
\begin{align*}
\begin{split}
    \widehat{\bm{w}}_{0} =& \ \underset{\bm{w}_0 \in \mathcal{W}_0}{\operatorname{argmin}} \  d_{\mathcal{M}}(\mathbbm{P}_{0,n}^{\bm{w}_0}, \mathbbm{P}_{n} ), \\
    \widehat{\bm{w}}_{1} =& \ \underset{\bm{w}_1 \in \mathcal{W}_1}{\operatorname{argmin}} \  d_{\mathcal{M}}(\mathbbm{P}_{1,n}^{\bm{w}_1}, \mathbbm{P}_{n} ),
\end{split}
\end{align*}
where $\mathcal{W}_0 \subseteq \mathcal{W}_0^{+}$ and $\mathcal{W}_1 \subseteq \mathcal{W}_1^{+}$ are the pre-specified set of weight vectors and $\mathcal{M}$ is the set of discriminators.

\paragraph{Parametric CBIPM for ATE}
Assume 
$$\operatorname{logit}(\pi(\cdot)) = c_0 + f( \ \cdot \ ; \bm{\theta}_0)$$
holds for unknown $c_0 \in \mathbb{R}$ and $\bm{\theta}_0 \in \Theta,$ where
$\Theta$ is a compact set of  $\mathbb{R}^k$ for  $k \in \mathbb{N}$ and
$f(\ \cdot \ ; \bm{\theta})$ is a function parameterized by $\bm{\theta}  \in \Theta.$ 
For the identifiability of the parameters, we assume
$f(\bm{0} ; \bm{\theta})=0$ for every $\bm{\theta} \in \Theta$.  
For $t \in \{ 0,1 \}$,  we consider
\begin{align*}
     \mathcal{W}_t^{P}(f) :=& \Big\{\bm{w}_{t,f}(\bm{\theta} \ ; \mathcal{D}^{(n)}): \bm{\theta} \in \Theta \Big\} ,   
\end{align*}
where $\bm{w}_{0,f}(\ \cdot \ ; \mathcal{D}^{(n)}) : \Theta \to \mathcal{W}_0^{+}$ and $\bm{w}_{1,f}(\ \cdot \ ; \mathcal{D}^{(n)}) : \Theta \to \mathcal{W}_1^{+}$ are n-dimensional vector functions defined as
\begin{align*}
     \bm{w}_{0,f}(\bm{\theta} ; \mathcal{D}^{(n)})_{(i)} :=& \frac{1}{n} + \frac{n_1}{n}\frac{\exp(f(\bm{X}_i ; \bm{\theta}))}{ \sum_{i : T_i = 0} \exp({f(\bm{X}_i ; \bm{\theta}))}} , && i : T_i = 0, \\
     \bm{w}_{1,f}(\bm{\theta} ; \mathcal{D}^{(n)})_{(i)} :=& \frac{1}{n} + \frac{n_0}{n}\frac{\exp(-f(\bm{X}_i ; \bm{\theta}))}{ \sum_{i : T_i = 1} \exp(-f(\bm{X}_i ; \bm{\theta}))} , && i : T_i = 1.
\end{align*}     

Finally, the parametric CBIPM method (P-CBIPM) for the ATE solves
\begin{align}
    \widehat{\bm{w}}_{t} = \underset{\bm{w}_t \in \mathcal{W}_t^{P}(f)}{\operatorname{argmin}} \  d_{\mathcal{M}}(\mathbbm{P}_{t,n}^{\bm{w}_t}, \mathbbm{P}_{n} ), \label{solve_P_CBIPM_ATE}
\end{align}
for $t \in \{ 0,1 \}$ and estimates the ATE using (\ref{ATE_estimator}).

For simplicity, we denote $f_{\bm{\theta}} (\cdot) := f(\cdot ; \bm{\theta})$ for the proof .
\begin{lemma} \label{lemma1_ATE}
For given $\bm{\theta} \in \Theta$, denote
\begin{align*}
d \mathbbm{P}^{\bm{\theta}}_0 (\bm{x}) :=& 
\left(\mathbbm{P}(T=0) + \mathbbm{P}(T=1)
\frac{\exp(f_{\bm{\theta}}(\bm{x})) }
{\int_{\bm{x} \in \mathcal{X}} \exp(f_{\bm{\theta}}(\bm{x})) d\mathbbm{P}_{0}(\bm{x})}\right)  d\mathbbm{P}_{0}(\bm{x}), \\
d\mathbbm{P}^{\bm{\theta}}_1 (\bm{x}) :=& 
\left(\mathbbm{P}(T=1) + \mathbbm{P}(T=0)
\frac{\exp(-f_{\bm{\theta}}(\bm{x})) }
{\int_{\bm{x} \in \mathcal{X}} \exp(-f_{\bm{\theta}}(\bm{x})) d\mathbbm{P}_{1}(\bm{x})}\right)  d\mathbbm{P}_{1}(\bm{x})
\end{align*}
as the weighted probability measure for the control and treated group. 
If $\Theta \subset \mathbb{R}^k$ is compact and Assumption \ref{assum_fconti}, \ref{assum_fiden} and \ref{assum_mzero} hold, 
then for $t \in \{0,1 \}$ and any $\epsilon>0$,
\begin{align}
\begin{split}
    \underset{|\bm{\theta} - \bm{\theta}_0 | \geq \epsilon}{\inf} d_{\mathcal{M}} (\mathbbm{P}_{t}^{\bm{\theta}}, \mathbbm{P}) > 0. \label{lemma1eq_ATE}
\end{split}
\end{align}    
\end{lemma}
\begin{proof}
For any $\bm{x} \in \mathcal{X}$, we obtain
\begin{align}
    d\mathbbm{P}^{\bm{\theta}_0}_0 (\bm{x}) = & 
    \mathbbm{P}(T=0) d\mathbbm{P}_{0}(\bm{x}) + \mathbbm{P}(T=1)
\frac{\exp(f_{\bm{\theta}_0}(\bm{x})) d\mathbbm{P}_{0}(\bm{x})}
{\int_{\bm{x} \in \mathcal{X}} \exp(f_{\bm{\theta}_0}(\bm{x})) d\mathbbm{P}_{0}(\bm{x})} \nonumber \\
=&  \mathbbm{P}(T=0) d\mathbbm{P}_{0}(\bm{x}) + \mathbbm{P}(T=1) d\mathbbm{P}_{1}(\bm{x}) \nonumber\\
=& d\mathbbm{P}(\bm{x}) \label{eqmeasure_ATE1},
\end{align}
and
\begin{align*}
d\mathbbm{P}^{\bm{\theta}_0}_1 (\bm{x}) = d\mathbbm{P}(\bm{x}).
\end{align*}

If $\mathbbm{P}_{0}^{\bm{\theta}} \equiv \mathbbm{P}$ holds for some $\bm{\theta} \in \Theta$, then
$$\frac{\exp(f_{\bm{\theta}}(\cdot)) }
{\mathbb{E}_0 (\exp(f_{\bm{\theta}}(\bm{X})))} \equiv \frac{\exp(f_{\bm{\theta}_0}(\cdot)) }
{\mathbb{E}_0 (\exp(f_{\bm{\theta}_0}(\bm{X})))},$$ 
and thus we get $f_{\bm{\theta}} \equiv f_{\bm{\theta}_0}$ since $f_{\bm{\theta}}(\bm{0}) = f_{\bm{\theta_0}}(\bm{0})=0$. 
Hence, by Assumption \ref{assum_fiden}, $\bm{\theta} = \bm{\theta}_0$.
Similarly, $\mathbbm{P}_{1}^{\bm{\theta}} \equiv \mathbbm{P}$ if and only if $\bm{\theta} = \bm{\theta}_0$.

To sum up, for any $\bm{\theta} \neq \bm{\theta}_0$, $d_{\mathcal{M}} (\mathbbm{P}_{t}^{\bm{\theta}}, \mathbbm{P}) > 0$ by Assumption \ref{assum_mzero}.
Since $\{\bm{\theta} \in \Theta : |\bm{\theta} - \bm{\theta}_0 | \geq \epsilon \}$ is compact and $ \bm{\theta} \mapsto d_{\mathcal{M}} (\mathbbm{P}_{t}^{\bm{\theta}}, \mathbbm{P})$ is continuous by Assumption \ref{assum_fconti}, we obtain (\ref{lemma1eq_ATE}).

\end{proof}

\begin{theorem} \label{theorem1_ATE}
Assume there exist unknown $c_0 \in \mathbb{R}$ and $\bm{\theta}_0 \in \Theta$ such that
$$ \operatorname{logit}(\pi(\cdot)) = c_0 + f_{\bm{\theta}_0} (\cdot).$$
If Assumptions \ref{assum_fconti}, \ref{assum_fiden}, \ref{assum_mfinite} and \ref{assum_mzero} hold, then for $\widehat{\bm{w}}_{0}$ and $\widehat{\bm{w}}_{1}$ defined by (\ref{solve_P_CBIPM_ATE}),
\begin{align*}
    \widehat{\operatorname{ATE}}^{\widehat{\bm{w}}_{0}, \widehat{\bm{w}}_{1}} \to \operatorname{ATE}.
\end{align*}
in probability.
\end{theorem}

\begin{proof}
For given $\bm{\theta} \in \Theta$, denote
\begin{align*}
d\mathbbm{P}^{\bm{\theta}}_0 (\bm{x}) :=& 
\left(\mathbbm{P}(T=0) + \mathbbm{P}(T=1)
\frac{\exp(f_{\bm{\theta}}(\bm{x})) }
{\int_{\bm{x} \in \mathcal{X}} \exp(f_{\bm{\theta}}(\bm{x})) d\mathbbm{P}_{0}(\bm{x})}\right)  d\mathbbm{P}_{0}(\bm{x}), \\
d\mathbbm{P}^{\bm{\theta}}_1 (\bm{x}) :=& 
\left(\mathbbm{P}(T=1) + \mathbbm{P}(T=0)
\frac{\exp(-f_{\bm{\theta}}(\bm{x})) }
{\int_{\bm{x} \in \mathcal{X}} \exp(-f_{\bm{\theta}}(\bm{x})) d\mathbbm{P}_{1}(\bm{x})}\right)  d\mathbbm{P}_{1}(\bm{x})
\end{align*}
as the weighted probability measure for the control and treated group. 
Also, we denote $\mathbbm{P}_{t,n}^{\bm{\theta}} := \mathbbm{P}_{t,n}^{\bm{w}_{t,f}(\bm{\theta} \ ; \mathcal{D}^{(n)})}$ and
 $\bm{w}_t^{\bm{\theta}} = (w^{\bm{\theta}}_{t,1} , \dots, w^{\bm{\theta}}_{t,n})^{\top} := \bm{w}_{t,f}(\bm{\theta} \ ; \mathcal{D}^{(n)})$ for $t \in \{ 0,1 \}$.
For any given $\epsilon>0$, there exist $\delta>0$ by Lemma \ref{lemma1_ATE} such that 
\begin{align*}
    \underset{|\bm{\theta} - \bm{\theta}_0 | \geq \epsilon}{\inf} d_{\mathcal{M}} (\mathbbm{P}_{0}^{\bm{\theta}}, \mathbbm{P}) = \delta > 0.
\end{align*}

Note that for $t \in \{ 0,1 \}$, solving (\ref{solve_P_CBIPM_ATE}) is identical to solving
$$
\widehat{\bm{\theta}}_{t,n} :=
\underset{\bm{\theta} \in \Theta}{\operatorname{argmin}} \
d_{\mathcal{M}}(\mathbbm{P}_{t,n}^{\bm{w}_{t,f}(\bm{\theta} \ ; \mathcal{D}^{(n)})}, \mathbbm{P}_{n} ).
$$
We will show that $\widehat{\bm{\theta}}_{t,n}$ converges to $\bm{\theta}_0$.
Since $d_{\mathcal{M}} (\mathbbm{P}_{t,n}^{\widehat{\bm{\theta}}_{t,n}}, \mathbbm{P}_{n}) \leq d_{\mathcal{M}} (\mathbbm{P}_{t,n}^{\bm{\theta}_0}, \mathbbm{P}_{n})$ by the definition of $\widehat{\bm{\theta}}_{t,n}$, we get
\begin{align}
    \mathbbm{P}_{\bm{\theta}_0}\left( |\widehat{\bm{\theta}}_{t,n} - \bm{\theta}_0 | \geq \epsilon \right)
    \leq& \mathbbm{P}_{\bm{\theta}_0}\left( \inf_{|\bm{\theta} - \bm{\theta}_0| \geq \epsilon} \left( d_{\mathcal{M}} (\mathbbm{P}_{t,n}^{\bm{\theta}}, \mathbbm{P}_{n}) - d_{\mathcal{M}} (\mathbbm{P}_{t,n}^{\bm{\theta}_0}, \mathbbm{P}_{n})\right)  \leq 0 \right) \nonumber \\
    \leq&  \mathbbm{P}_{\bm{\theta}_0}\left( \inf_{|\bm{\theta} - \bm{\theta}_0| \geq \epsilon} \left( d_{\mathcal{M}} (\mathbbm{P}_{t,n}^{\bm{\theta}}, \mathbbm{P}_{n}) - d_{\mathcal{M}} (\mathbbm{P}_{t,n}^{\bm{\theta}_0}, \mathbbm{P}_{n})\right)
    - \underset{|\bm{\theta} - \bm{\theta}_0 | \geq \epsilon}{\inf} d_{\mathcal{M}} (\mathbbm{P}_{t}^{\bm{\theta}}, \mathbbm{P}) < -\frac{\delta}{2} \right) \nonumber \\
    \leq & \mathbbm{P}_{\bm{\theta}_0}\left( \inf_{|\bm{\theta} - \bm{\theta}_0| \geq \epsilon} \left( d_{\mathcal{M}} (\mathbbm{P}_{t,n}^{\bm{\theta}}, \mathbbm{P}_{n}) - d_{\mathcal{M}} (\mathbbm{P}_{t,n}^{\bm{\theta}_0}, \mathbbm{P}_{n})
    - d_{\mathcal{M}} (\mathbbm{P}_{t}^{\bm{\theta}}, \mathbbm{P}) \right) < -\frac{\delta}{2} \right) \nonumber \\
    \leq&  \mathbbm{P}_{\bm{\theta}_0}\left( \sup_{\bm{\theta} \in \Theta} \left| d_{\mathcal{M}} (\mathbbm{P}_{t,n}^{\bm{\theta}}, \mathbbm{P}_{n}) - d_{\mathcal{M}} (\mathbbm{P}_{t}^{\bm{\theta}}, \mathbbm{P}) \right|  > \frac{\delta}{4} \right), \label{inprob_ATE1}
\end{align}
where the last inequality holds because $d_{\mathcal{M}} (\mathbbm{P}_{t}^{\bm{\theta}_0}, \mathbbm{P}_{1}) = 0$ by (\ref{eqmeasure_ATE1}).
Since
\begin{align*}
    &\left| d_{\mathcal{M}} (\mathbbm{P}_{t,n}^{\bm{\theta}}, \mathbbm{P}_{n}) - d_{\mathcal{M}} (\mathbbm{P}_{t}^{\bm{\theta}}, \mathbbm{P}) \right| \\
    =& \underset{m \in \mathcal{M}}{\sup} \Big| \sum_{i : T_i = 0} w_{t,i}^{\bm{\theta}} m(\bm{X}_i)  - \frac{1}{n} \sum_{i = 1}^n m(\bm{X}_i)  \Big|
     \quad - \underset{m \in \mathcal{M}}{\sup} \Big| \int_{\bm{x} }  m(\bm{x}) d \mathbbm{P}_t^{\bm{\theta}}(\bm{x}) - \int_{\bm{x} }  m(\bm{x}) d\mathbbm{P} (\bm{x}) \Big| \Bigg| \\
    \leq&  \underset{m \in \mathcal{M}}{\sup} \Big| \sum_{i : T_i = 0} w_{t,i}^{\bm{\theta}} m(\bm{X}_i) - \int_{\bm{x} }  m(\bm{x}) d\mathbbm{P}_t^{\bm{\theta}}(\bm{x}) \Big| 
     \quad + \underset{m \in \mathcal{M}}{\sup} \Big| \frac{1}{n} \sum_{i = 1}^n m(\bm{X}_i)  - \int_{\bm{x} }  m(\bm{x}) d\mathbbm{P} (\bm{x}) \Big|,
\end{align*}
we obtain
\begin{align*}
\mathbbm{P}_{\bm{\theta}_0}\left( \sup_{\bm{\theta} \in \Theta} \left| d_{\mathcal{M}} (\mathbbm{P}_{t,n}^{\bm{\theta}}, \mathbbm{P}_{n}) - d_{\mathcal{M}} (\mathbbm{P}_{t}^{\bm{\theta}}, \mathbbm{P}) \right|  > \frac{\delta}{4} \right)
\leq & \mathbbm{P}_{\bm{\theta}_0}\left( \sup_{\bm{\theta} \in \Theta}  \underset{m \in \mathcal{M}}{\sup} \Big| \sum_{i : T_i = 0} w_{t,i}^{\bm{\theta}} m(\bm{X}_i) - \int_{\bm{x} }  m(\bm{x})d \mathbbm{P}_t^{\bm{\theta}}(\bm{x}) \Big| > \frac{\delta}{8}  \right) \\
&+  \mathbbm{P}_{\bm{\theta}_0}\left( \underset{m \in \mathcal{M}}{\sup} \Big| \frac{1}{n} \sum_{i = 1}^n m(\bm{X}_i)  - \int_{\bm{x} }  m(\bm{x}) d\mathbbm{P} (\bm{x}) \Big|  > \frac{\delta}{8}  \right) \\
\to & 0
\end{align*}
by Assumption \ref{assum_mfinite} and Lemma \ref{lemma2}.

To sum up, (\ref{inprob_ATE1}) converges to $0$, which implies that $\widehat{\bm{\theta}}_{t,n}$ converges to $\bm{\theta}_0$ in probability.
Hence, by Assumption \ref{assum_fconti}, we obtain 
\begin{align*}
\widehat{\operatorname{ATE}}^{\widehat{\bm{w}}_{0}, \widehat{\bm{w}}_{1}}
=& \sum_{i : T_i = 1}   \left(\frac{1}{n} + \frac{n_0}{n}\frac{\exp(-f(\bm{X}_i ; \widehat{\bm{\theta}}_{1,n}))}{ \sum_{i : T_i = 1} \exp(-f(\bm{X}_i ; \widehat{\bm{\theta}}_{1,n}))}\right) Y_i- 
\sum_{i : T_i = 0}
\left(\frac{1}{n} + \frac{n_1}{n}\frac{\exp(f(\bm{X}_i ; \widehat{\bm{\theta}}_{0,n}))}{ \sum_{i : T_i = 0} \exp({f(\bm{X}_i ; \widehat{\bm{\theta}}_{0,n}))}}\right) Y_i \\
\to& \mathbb{E}\left(Y(1)\right) - \mathbb{E}\left(Y(0)\right) \\
=& \operatorname{ATE}.
\end{align*}
\end{proof}

\paragraph{Nonparametric CBIPM for ATE}
We consider
\begin{align*}
     \mathcal{W}_0^N (B) :=& \Big\{\bm{w}_0 \in \mathcal{W}_0^{+} : \max_{i \in [n]} w_i \leq \frac{B}{n_0} \Big\}, \\
     \mathcal{W}_1^N (B) :=& \Big\{\bm{w}_1 \in \mathcal{W}_1^{+} : \max_{i \in [n]} w_i \leq \frac{B}{n_1} \Big\},     
\end{align*}
where  $B > 0$ is a sufficiently large number such that $1/\eta \leq B.$ 
Then, we solve
\begin{align}
\begin{split}
    \widehat{\bm{w}}_{0} = \underset{\bm{w}_0 \in \mathcal{W}_0^N (B)}{\operatorname{argmin}} \  d_{\mathcal{M}}(\mathbbm{P}_{0,n}^{\bm{w}_0}, \mathbbm{P}_{n} ), \\
    \widehat{\bm{w}}_{1} = \underset{\bm{w}_1 \in \mathcal{W}_1^N (B)}{\operatorname{argmin}} \  d_{\mathcal{M}}(\mathbbm{P}_{1,n}^{\bm{w}_1}, \mathbbm{P}_{n} ), \label{solve_N_CBIPM_ATE}
\end{split}
\end{align}
and estimate the ATE by (\ref{ATE_estimator}).

\begin{assumpM} ($\mathcal{M}$ dominates $\mathcal{M}_0$) \label{assum_msub_ATE}
    There exists a class $\mathcal{M}_0$ of outcome regression models including 
    $m_0(\cdot)$ and $m_1(\cdot)$ such that
    $d_{\mathcal{M}_0}(\cdot, \cdot) \lesssim d_{\mathcal{M}}(\cdot, \cdot)$.
\end{assumpM}

\begin{theorem} \label{theorem2_ATE}
Consider $\mathcal{M}$ which satisfies Assumptions \ref{assum_mfinite} and \ref{assum_msub_ATE}.
Then for $\widehat{\bm{w}}_{0}$ and $\widehat{\bm{w}}_{1}$ defined by (\ref{solve_N_CBIPM_ATE}),
\begin{align*}
    \widehat{\operatorname{ATE}}^{\widehat{\bm{w}}_{0}, \widehat{\bm{w}}_{1}} \to \operatorname{ATE}.
\end{align*}
in probability.
\end{theorem}
\begin{proof}

Let $u_0(\cdot) := \frac{1}{1-\pi(\cdot)}$ and $u_1(\cdot) := \frac{1}{\pi(\cdot)}$. 
For given $\mathcal{D}_n$ and $t \in \{ 0,1 \}$, we define
$\bm{w}_t^{0} = (w_{t,1}^{0}, \dots, w_{t,n}^{0})^{\top}$ as 
\begin{align*}
    w_{t,i}^{0} = \frac{\mathbb{I}(T_i = t) u_t(\bm{X}_i)}{ \sum_{i : T_i = t} {u_t(\bm{X}_i)}}, && i \in [n].
\end{align*}
Since
$$ 1 < \frac{1}{1-\eta} \leq u_t(\cdot) \leq \frac{1}{\eta},$$
$\bm{w}_t^{0} \in \mathcal{W}_t^N (B)$ holds. By definition of $\widehat{\bm{w}}_t$,
\begin{align}
    d_{\mathcal{M}}(\mathbbm{P}_{t,n}^{\widehat{\bm{w}}_t}, \mathbbm{P}_{n}) \leq  d_{\mathcal{M}}(\mathbbm{P}_{t,n}^{\bm{w}_t^{0}}, \mathbbm{P}_{n}). \label{smallIPM_ATE}
\end{align}

For any $m(\cdot)\in \mathcal{M}$ we have
\begin{align}
\sup_{m\in \mathcal{M}} 
\left| \sum_{i= 1}^n \frac{m(\bm{X}_i)}{n} - 
\mathbb{E}(m(\bm{X})) \right| \to 0 \label{asconv1_ATE}
\end{align}
by Lemma \ref{lemma2}.
Also, since
\begin{align*}
\mathbb{E} (\mathbb{I}(T = t) u_t (\bm{X}) m(\bm{X}) ) 
= \mathbb{E} ( m(\bm{X}))
\end{align*}
and
\begin{align*}
\mathbb{E} (\mathbb{I}(T = t) u_t (\bm{X})) = 1,
\end{align*}
we obtain
\begin{align*}
    \sum_{i : T_i = t} w_{t,i}^{0} m(\bm{X}_i)
=& \frac{\sum_{i : T_i = t} u_t(\bm{X}_i) m(\bm{X}_i)}{\sum_{i : T_i = t} u_t(\bm{X}_i)}\\
=& \frac{\sum_{i=1}^n \mathbb{I}(T=t) u_t(\bm{X}_i) m(\bm{X}_i)}{\sum_{i=1}^n \mathbb{I}(T=t) u_t(\bm{X}_i)}\\
= & \frac{\frac{1}{n}\sum_{i=1}^n \mathbb{I}(T=t) u_t(\bm{X}_i) m(\bm{X}_i)}{\frac{1}{n}\sum_{i=1}^n \mathbb{I}(T=t) u_t(\bm{X}_i)} \\
\to & \frac{\mathbb{E} (\mathbb{I}(T = t) u_t(\bm{X}) m(\bm{X}) ) }{\mathbb{E} (\mathbb{I}(T = t) u_t(\bm{X})) } \\
= & \mathbb{E}(m(\bm{X})). 
\end{align*}
Hence, 
\begin{align}
    \sup_{m\in \mathcal{M}} 
\left| \sum_{i : T_i = t} w_{t,i}^{0} m(\bm{X}_i) - 
\mathbb{E}(m(\bm{X})) \right| \to 0 \label{asconv2_ATE}
\end{align}
holds by Lemma \ref{lemma2}.

To sum up, for $t \in \{ 0,1 \}$,
\begin{align*}
    d_{\mathcal{M}}(\mathbbm{P}_{t,n}^{\widehat{\bm{w}}}, \mathbbm{P}_{n}) \leq & d_{\mathcal{M}}(\mathbbm{P}_{t,n}^{\bm{w}_t^{0}}, \mathbbm{P}_{n}) \\
    = & \sup_{m\in \mathcal{M}} 
 \left| \sum_{i : T_i = 0} w_{t,i}^{0} m(\bm{X}_i) - \sum_{i = 1}^n \frac{m(\bm{X}_i)}{n} \right| \\
\to &  0 
\end{align*}
holds by (\ref{smallIPM_ATE}), (\ref{asconv1_ATE}) and (\ref{asconv2_ATE}).
Since $d_{\mathcal{M}_0}(\cdot, \cdot) \lesssim d_{\mathcal{M}}(\cdot, \cdot)$ by Assumption \ref{assum_msub}, we obtain
\begin{align}
\operatorname{Er}_{\text{bal}}^{\widehat{\bm{w}}_0, \widehat{\bm{w}}_1}
= & \left( \sum_{i= 1}^n \frac{m_0 (\bm{X}_i)}{n} - \sum_{i : T_i = 0} \widehat{w}_{0,i} m_0 (\bm{X}_i) \right) 
     - \left(\sum_{i= 1}^n \frac{m_1 (\bm{X}_i)}{n} - \sum_{i : T_i = 1} \widehat{w}_{1,i} m_1 (\bm{X}_i) \right) \nonumber\\
\leq & d_{\mathcal{M}_0}(\mathbbm{P}_{0,n}^{\widehat{\bm{w}}_0}, \mathbbm{P}_{n})
+ d_{\mathcal{M}_0}(\mathbbm{P}_{1,n}^{\widehat{\bm{w}}_1}, \mathbbm{P}_{n}) \nonumber \\
\leq & \xi(d_{\mathcal{M}}(\mathbbm{P}_{0,n}^{\widehat{\bm{w}}_0}, \mathbbm{P}_{n}))
+ \xi(d_{\mathcal{M}}(\mathbbm{P}_{1,n}^{\widehat{\bm{w}}_1}, \mathbbm{P}_{n})) \nonumber \\
\to & 0. \label{balzero_ATE}
\end{align}

Let $\mathcal{C}^{(n)} = \{(\bm{X}_i, T_i)\}_{i=1}^n$.
Since $\widehat{w}_{t,i}$ is a measurable function of $\mathcal{C}^{(n)}$ for every $i \in [n]$ and $t \in \{ 0,1 \}$, we can easily show
\begin{align} \operatorname{Er}_{\text{obs}}^{\widehat{\bm{w}}_0, \widehat{\bm{w}}_1} =&
-  \sum_{i : T_i = 0} w_{0,i} (Y_i - m_0(\bm{X}_i)) + \sum_{i : T_i = 1} w_{1,i} (Y_i - m_1(\bm{X}_i)) 
\to 0 \label{obszero_ATE}
\end{align}
in probability, by following the proof of (\ref{obszero}).
Now, by (\ref{balzero_ATE}), (\ref{obszero_ATE}) and the fact $\operatorname{SATE} \to \operatorname{ATE}$ in probability, we conclude $\widehat{\operatorname{ATE}}^{\widehat{\bm{w}}_0 , \widehat{\bm{w}}_1} \to \operatorname{ATE}$ in probability.

\end{proof}

\subsection{Experiments}
\begin{table*}[h]
\renewcommand{\arraystretch}{1.3}
\caption{\textbf{Kang-Schafer example with the linear/nonlinear propensity scores, for the ATE.} 
\label{table_ATE}
We generate 1000 simulations and report the bias and RMSE as the performance measures.
For each pair of dataset and performance measure (i.e. for each row), the two best results are marked by bold letters.}
\centering
\scalebox{1.0}{\small
\begin{tabular}{|c|c|c|ccc|ccc|ccc|}
\hline
\multirow{2}{*}{$\operatorname{logit}(\pi(\cdot))$} & \multirow{2}{*}{Measure} &  \multirow{2}{*}{n} & \multicolumn{3}{c|}{Existing methods} & \multicolumn{3}{c|}{P-CBIPM} & \multicolumn{3}{c|}{N-CBIPM} \\
   \cline{4-12}
& &   & GLM & Boost & \multicolumn{1}{c|}{CBPS} & Wass & MMD & \multicolumn{1}{c|}{SIPM} & Wass & MMD & SIPM \\
    \hline \hline 
\multirow{4}{*}{Linear} & \multirow{2}{*}{Bias}  & 200 & 0.287 & 0.873 & -0.108 & -0.208 & \textbf{0.016} & \textbf{0.093}  & -0.222 & 0.638 & -0.847 \\
\cline{3-12}
& & 1000 & -0.059 & 0.497 & -0.069 & \textbf{-0.053} & -0.011 & \textbf{-0.015} & -0.409 & 0.448 & -1.323 \\
\cline{2-12}
& \multirow{2}{*}{RMSE} & 200 & 3.387 & 3.037 & 2.885 & 2.793 & 2.635 & 2.652 & \textbf{2.622} & \textbf{2.548} & 2.934 \\
\cline{3-12}
& & 1000 & 1.994 & \textbf{1.190} & 1.408 & 1.240 & 1.219 & 1.209 & 1.261 & \textbf{1.063} & 1.874 \\
\hline \hline
\multirow{4}{*}{Nonlinear} & \multirow{2}{*}{Bias}  & 200 & \textbf{-1.848} & -10.624 & -5.116 & -4.934 & -4.875 & -4.929 & -4.563 & \textbf{-2.568} & -3.461  \\
\cline{3-12}
& & 1000 & \textbf{1.606} & -7.585 & -6.012 & -5.325 & -5.261 & -5.101 & -4.707 & \textbf{-2.827} & -3.737  \\
\cline{2-12}
& \multirow{2}{*}{RMSE} & 200 & 9.462 & 11.101 & 5.995 & 5.582 & 5.605 & 5.634 & 5.305 & \textbf{3.720} & \textbf{4.446} \\
\cline{3-12}
& & 1000 & 11.079 & 7.680 & 6.268 & 5.531 & 5.447 & 5.263 & 4.874 & \textbf{3.026} & \textbf{4.028} \\
\hline
\end{tabular}
}
\end{table*}

\newpage

Table \ref{table_ATE} presents the bias and RMSE for the ATE estimators. Generally, the results are similar 
to those for the ATT. An exception is that the biases of GLM for the nonlinear propensity score are small, which
we think occurs by chance since the RMSEs are very large.
In addition, we draw the boxplots of the estimated ATE values in Figure \ref{proposal_compare}
as the compliments to the results of Table \ref{table_ATE}.

\begin{figure}[h]
\centering
\subfigure[Linear $\operatorname{logit}(\pi(\cdot))$, $n=200$]{\includegraphics[width=0.48\linewidth]{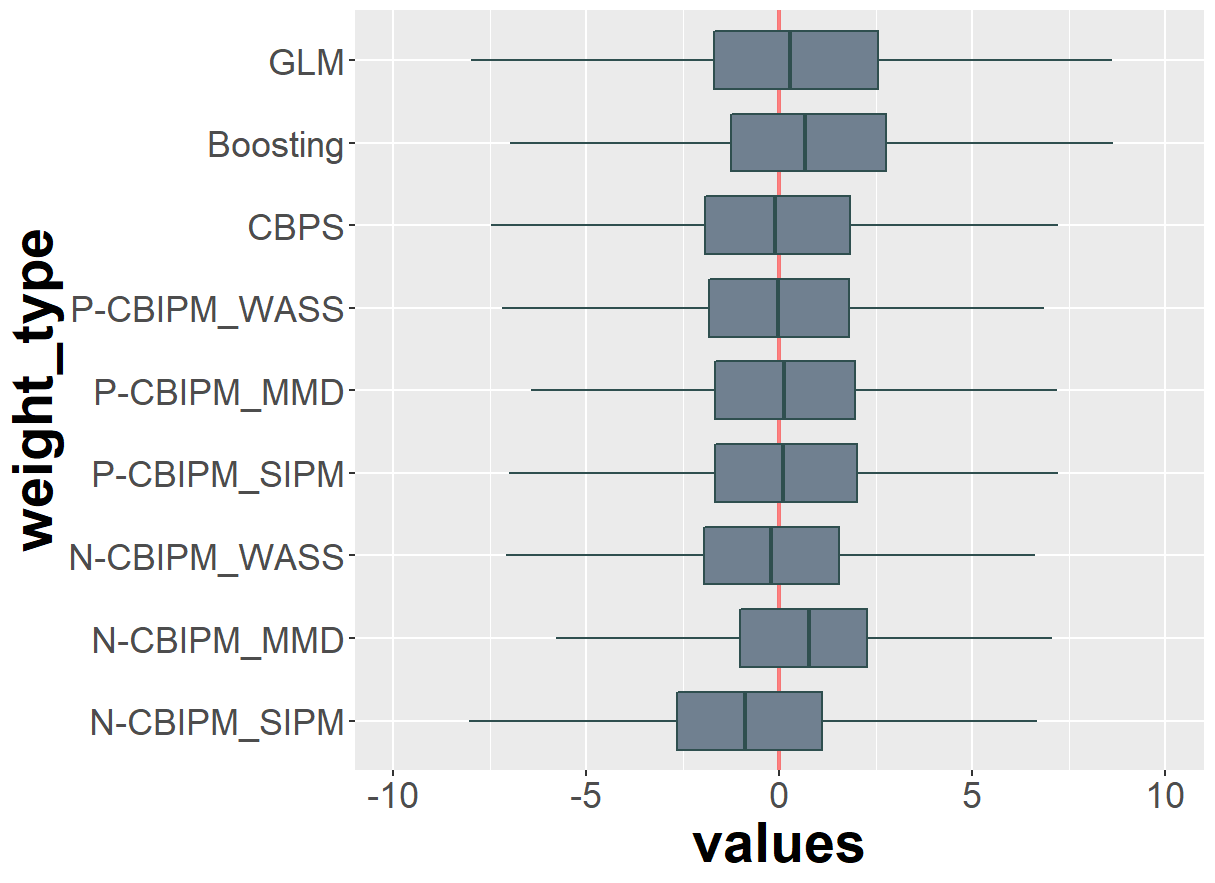}}
\hfill
\subfigure[Linear $\operatorname{logit}(\pi(\cdot))$, $n=1000$]{\includegraphics[width=0.48\linewidth]{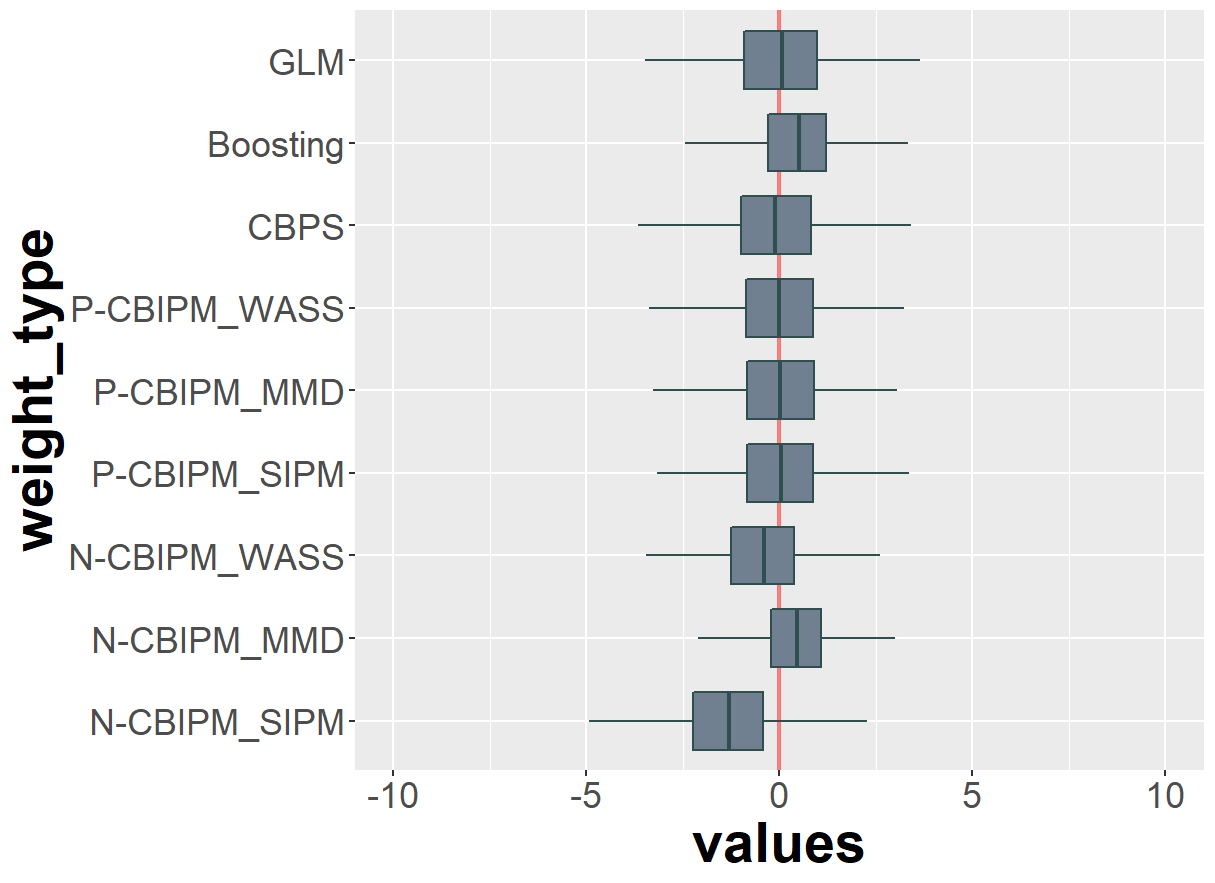}}
\hfill
\subfigure[Nonlinear $\operatorname{logit}(\pi(\cdot))$, $n=200$]{\includegraphics[width=0.48\linewidth]{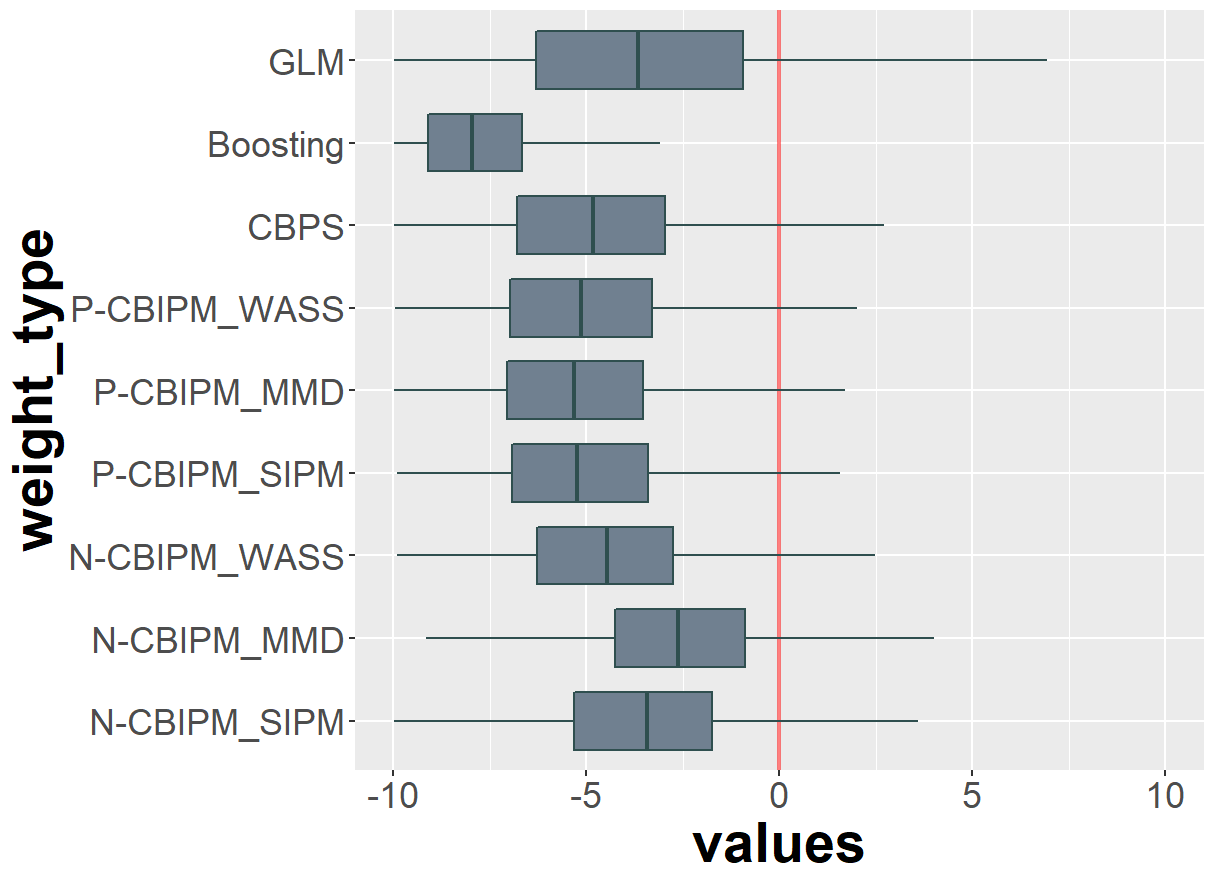}}
\hfill
\subfigure[Nonlinear $\operatorname{logit}(\pi(\cdot))$, $n=1000$]{\includegraphics[width=0.48\linewidth]{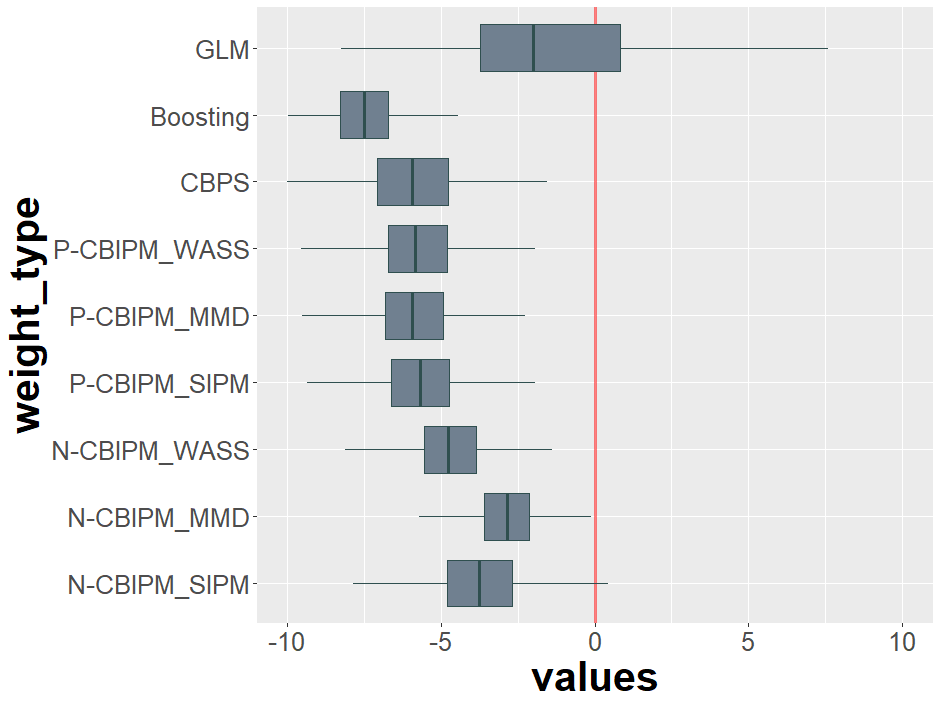}}
\caption{\textbf{Boxplots of the estimated ATE values on  Kang-Schafer example.} 
The boxplots are drawn based on the ATE estimates obtained from 1000 simulated datasets for the Kang-Schafer example. 
X-axis and Y-axis represent the estimated ATE values and the weighting methods, respectively.
} \label{proposal_compare}
\end{figure}

\newpage
\section{Experimental details} \label{app_C}
We use R (ver. 4.0.2), Python (ver. 3.6), and NVIDIA TITAN Xp GPUs to obtain the estimates of the ATT and the ATE. 
Detailed settings for each method are as follows:

\paragraph{GLM} We use linear logistic regression, where the regression coefficients are estimated by the MLE.

\paragraph{Boosted CART}
We follow the experimental setting in \citet{lee2010improving}.
To be more specific, we use twang package \cite{ridgeway2017toolkit}, with 20,000 iterations and the shrinkage parameter of 0.0005, with the iteration stopping point that minimizes the mean of the Kolmogorov-Smirnov test statistics.

\paragraph{CBPS} CBPS is implemented using CBPS package \cite{fong2022package} with the default parameters.

\paragraph{EB} EB is implemented using EB package \cite{hainmueller2022package} with the default parameters.

\paragraph{The CBIPM with the Wassersterin distance} 
We implement the CBIPM with the Wasserstein distance using techniques suggested by \citet{arjovsky2017wasserstein} and \citet{gulrajani2017improved}. 
    More specifically, we randomly sample $\tilde{\bm{x}}_1, \dots, \tilde{\bm{x}}_R$ uniformly from along straight lines between pairs of points sampled from treated and control groups.
Then, on these samples, we calculate the gradient of the discriminator with respect to its input and penalize the norm of the gradient to go towards $1$.
Formula of $\mathcal{L}_{\textup{adv}}(\bm{\theta}, \bm{\psi})$ and $\mathcal{L}(\bm{\theta}, \bm{\psi})$ for Algorithm \ref{algATT} is as follow: 
\begin{align*}
\bm{x}^{(c)}_{r} &\sim \mathbbm{P}_{0,n}, \ \bm{x}^{(t)}_{r} \sim  \mathbbm{P}_{1,n}, \ u_r \sim \mathcal{U}(0,1), \ \ \ \ \quad r \in [R],\\
\tilde{\bm{x}}_r &:= u_r \bm{x}^{(c)}_{r} + (1-u_r) \bm{x}^{(t)}_{r}, \qquad \qquad \qquad  r \in [R],\\
\mathcal{L}_{\textup{adv}}(\bm{\theta}, \bm{\psi}) &:= \ell_n(\bm{\theta}, \bm{\psi}) 
- \frac{\tau}{R} \sum_{r=1}^R ( || \nabla_{\tilde{\bm{x}}_r} 
 m(\tilde{\bm{x}}_r ; \bm{\psi}) ||_2 -1 )^2, \\
\mathcal{L}(\bm{\theta}, \bm{\psi}) &:= \ell_n(\bm{\theta}, \bm{\psi}),
\end{align*}
where $m(\cdot ; \bm{\psi})$ is a neural network parameterized by $\bm{\psi},$ and $\tau$ is regularization parameters. 
For both the P-CBIPM and the N-CBIPM, we use a neural network with 100 hidden nodes with leaky relu. 
We use Adam \cite{kingma2014adam} optimizer with $\textup{lr} = 0.03$ and $T=1000$ for gradient descent steps, and Adam optimizer with $\textup{lr}_{\textup{adv}} = 0.3$, $T_{\text{adv}} = 5$ for gradient ascent steps.
$\tau = 0.3$ and $R=100$ are used.    
For additional stability, we clip the weights and biases of $m(\cdot ; \bm{\psi})$ to $0.1$ after each gradient ascent step.

\paragraph{The CBIPM with MMD}   
RBF kernel $k_{\gamma}: \mathbb{R}^d \times \mathbb{R}^d \to \mathbb{R}$ is defined as
\begin{align*}
    k_{\gamma}(\bm{x},\bm{x}^{\prime}) = \exp\left(-\frac{\|\bm{x} - \bm{x}^{\prime}\|_2^2}{\gamma^2}\right).
\end{align*}
For $\bm{w} = (w_1 , \dots, w_n)^{\top} \in \mathcal{W}^{+}$, $d_{\mathcal{M}_{k_\gamma,B}}(\mathbbm{P}_{0,n}^{\bm{w}}, \mathbbm{P}_{1,n} )^2$ is expressed as
\begin{align*}
d_{\mathcal{M}_{k_\gamma,B}}(\mathbbm{P}_{0,n}^{\bm{w}}, \mathbbm{P}_{1,n} )^2 = \sum_{i : T_i = 0 } \sum_{j : T_j = 0 } w_i w_j k_{\gamma}(\bm{X}_i , \bm{X}_j) 
- 2 \sum_{i : T_i = 0 } \sum_{j : T_j = 1 } \frac{w_i}{n_1} k_{\gamma}(\bm{X}_i , \bm{X}_j)
+ \sum_{i : T_i = 1 } \sum_{j : T_j = 1 } \frac{1}{n_1^2}  k_{\gamma}(\bm{X}_i , \bm{X}_j).
\end{align*}

In turn, formula of $\mathcal{L}(\bm{\theta}, \bm{\psi})$ for Algorithm \ref{algATT} is
\begin{align*}
\mathcal{L}(\bm{\theta}, \bm{\psi}) = \mathcal{L}(\bm{\theta}) = d_{\mathcal{M}_{k_\gamma,B}}(\mathbbm{P}_{0,n}^{\bm{w}(\bm{\theta})}, \mathbbm{P}_{1,n} )^2,
\end{align*}
and we don't need line $2\sim4 $ of Algorithm \ref{algATT}.
We use Adam optimizer with $\textup{lr} = 0.03$, $T=1000$ for gradient descent steps, and $\gamma = 10$.

\paragraph{The CBIPM with the SIPM}
Since $\mathcal{M}_{sig}$ is a parameter family, we can simply iterate gradient ascent steps and descent steps.  
However, we confirm that mode collapse \cite{salimans2016improved, che2019mode} occurs frequently when we use the sigmoid IPM for covariate balancing.
To resolve this difficulty, we ensemble the multiple discriminators.

In the gradient ascent step, we update $\bm{\psi} = (\bm{\psi}_1 , \dots, \bm{\psi}_S)$ using the sum of the loss functions.
In the gradient descent step, we update $\bm{\theta}$ using the discriminator which has the highest loss value.
The formula of $\mathcal{L}_{\textup{adv}}(\bm{\theta}, \bm{\psi})$ and $\mathcal{L}(\bm{\theta}, \bm{\psi})$ for Algorithm \ref{algATT} are
\begin{align*}
\mathcal{L}_{\textup{adv}}(\bm{\theta}, \bm{\psi}) :=& \sum_{s=1}^S \ell_n (\bm{\theta}, \bm{\psi}_s), \\
\mathcal{L}(\bm{\theta}, \bm{\psi}) :=& \max_{s\in[S]} \ell_n (\bm{\theta}, \bm{\psi}_s),
\end{align*}
where $\bm{\psi}_s = (\bm{\rho}_s^{\top},\mu_s)^{\top}$ and
$m(\bm{x} ; \bm{\psi}_s) = \sigma(\bm{\rho}_s^\top \bm{x} +\mu_s)$.

In Appendix \ref{sec_nsipm}, we show that using multiple discriminators dramatically improves the accuracies without increasing computing time much. 
That's because discriminators $m(\cdot ; \bm{\psi}_1), \dots, m(\cdot ; \bm{\psi}_S) \in \mathcal{M}_{sig}$ can be expressed using single linear function from $\mathbb{R}^d$ to $\mathbb{R}^S$, and hence parallel calculations using gpu can be used.

For P-CBIPM, we use Adam optimizer with $lr=0.03$, $T=1000$ for gradient descent steps, SGD optimizer with $\textup{lr}_{\textup{adv}} = 0.01$, and $T_{\text{adv}} = 1$ for gradient ascent steps.
For N-CBIPM, we use Adam optimizer with $lr=0.1$, $T=1000$ for gradient descent steps, SGD optimizer with $\textup{lr}_{\textup{adv}} = 1.0$, and $T_{\text{adv}} = 3$ for gradient ascent steps.
For both methods, we use $S=100$.

\newpage
\section{Comparison with doubly robust estimator}
\renewcommand{\theequation}{D.\arabic{equation}}
The weighting methods can be combined with an estimation of the outcome regression model to become doubly robust.
The augmented IPW (AIPW) \cite{robins1994estimation} is such an approach.
The AIPW first estimates the outcome regression model $\widehat{m}_0$ using only control samples and 
then applies the IPW to $Y_i- \widehat{m}_0(\bm{X}_i)$.
For general $\bm{w} \in \mathcal{W}^{+}$, the augmented estimator for the ATT can be expressed as
\begin{equation*}
    \widehat{\operatorname{ATT}}^{\bm{w}}_{\operatorname{Aug}} = \sum_{i : T_i = 1} \frac{Y_i - \widehat{m}_0(\bm{X}_i)}{n_1}-\sum_{i : T_i = 0} w_i (Y_i- \widehat{m}_0(\bm{X}_i)).
\end{equation*}
Similar with (\ref{decom}), the error of $\widehat{\operatorname{ATT}}^{\bm{w}}_{\operatorname{Aug}}$ can be decomposed as
\begin{align*}
    \widehat{\operatorname{ATT}}^{\bm{w}}_{\operatorname{Aug}} - \operatorname{ATT} 
    =  (\operatorname{err}_{\text{bal}}^{\bm{w}})_{\operatorname{Aug}} + \operatorname{err}_{\text{obs}}^{\bm{w}} + (\operatorname{SATT} - \operatorname{ATT}),
\end{align*}
where $\operatorname{err}_{\text{obs}}^{\bm{w}}$ is defined same as before and 
\begin{align*}
    (\operatorname{err}_{\text{bal}}^{\bm{w}})_{\operatorname{Aug}} =&  \sum_{i : T_i = 1} \frac{m_0 (\bm{X}_i) - \widehat{m}_0(\bm{X}_i)}{n_1}
   - \sum_{i : T_i = 0} w_i (m_0 (\bm{X}_i)-\widehat{m}_0(\bm{X}_i)).
\end{align*}
Note that $(\operatorname{err}_{\text{bal}}^{\bm{w}})_{\operatorname{Aug}}$ is close to zero no matter what $\bm{w}$ is used
when $\widehat{m}_0 \approx m_0$ and thus the augmented estimator is doubly robust. 
That is, the augmented estimator is consistent without modeling the weights correctly. In this sense,
the augmented estimator is similar to the N-CBIPM estimator.
However, the two methods work quite differently.
The key difference is that the balancing process of the CBIPM only uses $\{(X_i, T_i)\}_{i=1}^n$, 
but the outcomes $\{Y_i\}_{i=1}^n$ are also needed for augmentation. 
With pre-calculated weights only using $\{(X_i, T_i)\}_{i=1}^n$, the CBIPM can be used more flexibly 
in practice such as the time-varying outcome regression model situations. 

More important advantage of the N-CBIPM over augmentation is
that the accuracy of the N-CBIPM estimator is less influenced by the complexity
of estimating $m_0$ because the N-CBIPM does not use the outcomes when it estimates the weights.
To confirm this conjecture, we do a small experiment to compare the weighting methods with and without augmentation.
For augmentation, we use the ordinary least square estimator (OLS) and Bayesian additive regression tree
(BART) of \citet{chipman2010bart}.  To control the difficulty of estimation of $m_0,$
we use the two values (1 and 10) for the standard deviation of the noise in the Kang-Schafer example.

\begin{table*}[h!]
\renewcommand{\arraystretch}{1.3}
\caption{\textbf{Kang-Schafer example with small/large noises.} 
We generate 1000 simulations and report the bias and RMSE as the performance measures.
For each pair of dataset and performance measure (i.e. for each of three rows), the two best results are marked in bold letters.} \label{table2}
\centering
\scalebox{1.0}{\small
\begin{tabular}{|c|c|c|cccc|ccc|ccc|}
\hline
\multirow{2}{*}{std of noise} & \multirow{2}{*}{Measure} & \multirow{2}{*}{Aug} & \multicolumn{4}{c|}{Existing methods} & \multicolumn{3}{c|}{P-CBIPM} & \multicolumn{3}{c|}{N-CBIPM} \\
   \cline{4-13}
& &   & GLM & Boost & CBPS & \multicolumn{1}{c|}{EB} & Wass & MMD & \multicolumn{1}{c|}{SIPM} & Wass & MMD & SIPM \\
    \hline \hline 
\multirow{6}{*}{1} & \multirow{3}{*}{Bias}  & $\times$ & -7.233 & -8.375 & -4.745 &  -4.806 & -5.015 & -4.869 & -5.086 & -3.945 & \textbf{-2.732} & \textbf{-3.569} \\
& & OLS & -6.233 & -6.095 & -4.739 & -4.755 & -4.885 & -4.830 & -4.979 & -3.939 & \textbf{-2.686} & \textbf{-3.347} \\
& & BART & -3.868 & -3.792 & -3.688 & -3.687 & -3.716 & -3.699 & -3.721 & -3.619 & \textbf{-3.519} & \textbf{-3.582} \\
\cline{2-13}
& \multirow{3}{*}{RMSE} & $\times$ & 8.275 & 9.204 & 5.354 & 5.395 & 5.681 & 5.455 & 5.697 & 4.632 & \textbf{4.422} & \textbf{4.491} \\
& & OLS & 7.161 & 6.898 & 5.344 & 5.352 & 5.518 & 5.422 & 5.593 & 4.624 & \textbf{4.394} & \textbf{4.223} \\
& & BART & 4.746 & 4.667 & 4.520 & 4.517 & 4.544 & 4.532 & 4.554 & 4.436 & \textbf{4.397} & \textbf{4.400} \\
\hline \hline 
\multirow{6}{*}{10} & \multirow{3}{*}{Bias}  & $\times$ & -7.164 & -8.306 & -4.671 & -4.732 & -4.942 & -4.795 & -4.994 & -3.886 & \textbf{-2.739} & \textbf{-3.517}  \\
& & OLS & -6.157 & -6.005 & -4.664 & -4.681 & -4.812 & -4.755 & -4.903 & -3.88 & \textbf{-2.693} & \textbf{-3.296}\\
& & BART & -4.698 & -4.516 & -4.250 & -4.251 & -4.291 & -4.279 & -4.320 & -4.057 & \textbf{-3.802} & \textbf{-3.961} \\
\cline{2-13}
& \multirow{3}{*}{RMSE} & $\times$ & 8.387 & 9.289 & 5.575 & 5.615 & 5.886 & 5.669 & 5.897 & \textbf{4.997} & 5.237 & \textbf{5.016} \\
& & OLS & 7.294 & 7.041 & 5.566 & 5.574 & 5.733 & 5.636 & 5.795 & \textbf{4.988} & 5.212 & \textbf{4.812} \\
& & BART & 5.942 & 5.757 & 5.405 & 5.402 & 5.457 & 5.432 & 5.480 & \textbf{5.209} & 5.290 & \textbf{5.210} \\
\hline
\end{tabular}
}
\end{table*}

The results with the nonlinear propensity score
are presented in Table \ref{table2}.
while it is helpful when the variance of the noise is small,
the augmentation using BART depreciates the performance of the N-CBIPM. 
That is, augmentation is only helpful when $m_0$ is easy to estimate.
Also, note that the N-CBIPM outperforms the other weighting methods with large margins
even with augmentation.

\newpage
\section{Additional experimental results}
\subsection{Results on other simulation designs} \label{App_another}

We consider various simulations models whose results are presented in this section.

\paragraph{Kang-Schafer example with small overlap}
To verify the CBIPM also works well for a case of small overlap, we modify the Kang-Schafer example as follows.
we generate the binary treatment indicators $T_i \in \{ 0,1 \}$ from 
\begin{align*}
\mathbbm{P}(T=1|\bm{Z}_i) = \sigma(-2 Z_{i1} + Z_{i2} -0.5 Z_{i3} -0.2 Z_{i4}).
\end{align*} in Kang-Schafer example,
which is obtained by multiplying 2 to the logit of the propensity score
of the original Kang-Schafer example. 
By multiplying 2, we make the overlap between data of trained and controlled groups
smaller.
Table \ref{table_overlap} presents the bias and RMSE for the ATT and the ATE estimators
for this model, which amply show that the N-CBIPMs outperform the other methods with large margins in terms of both the bias and RMSE.

\begin{table*}[h!]
\renewcommand{\arraystretch}{1.3}
\caption{\textbf{Kang-Schafer example with small overlap} 
We generate 1000 simulations and report bias and RMSE.
For each dataset and measure, the two best values are marked in bold letters.
We omit EB because its solution often doesn't converge in the small overlap situation.} \label{table_overlap}
\centering
\scalebox{1.0}{\small
\begin{tabular}{|c|c|c|ccc|ccc|ccc|}
\hline
\multirow{2}{*}{Interest} & \multirow{2}{*}{Measure} &  \multirow{2}{*}{n} & \multicolumn{3}{c|}{Existing methods} & \multicolumn{3}{c|}{P-CBIPM} & \multicolumn{3}{c|}{N-CBIPM} \\
   \cline{4-12}
& &   & GLM & Boost & \multicolumn{1}{c|}{CBPS} & Wass & MMD & \multicolumn{1}{c|}{SIPM} & Wass & MMD & SIPM \\
    \hline \hline 
\multirow{4}{*}{ATT} & \multirow{2}{*}{Bias}  & 200 & -12.708 & -14.277 & -8.586 & -10.836 & -9.833 & -8.709 & \textbf{-6.381} & \textbf{-6.940} & -8.139\\
\cline{3-12}
& & 1000 & -13.021 & -11.590 & -6.676 & -7.780 & -7.034 & -6.836 & \textbf{-4.546} & \textbf{-5.280} & -5.303 \\
\cline{2-12}
& \multirow{2}{*}{RMSE} & 200 & 14.178 & 15.186 & 9.638 & 12.322 & 11.006 & 9.614 & \textbf{7.752} & \textbf{8.248} & 10.699 \\
\cline{3-12}
& & 1000 & 13.391 & 12.169 & 6.927 & 8.379 & 7.423 & 6.973 & \textbf{4.708} & 5.563 & \textbf{5.559} \\
\hline \hline
\multirow{4}{*}{ATE} & \multirow{2}{*}{Bias}  & 200 & \textbf{-2.683} & -18.393 & -8.292 & -9.186 & -9.517 & -9.456 & -7.766 & \textbf{-4.836} & -5.526 \\
\cline{3-12}
& & 1000 & 6.785 & -14.486 & -9.447 & -9.628 & -9.843 & -9.687 & -7.991 & \textbf{-5.041} & \textbf{-4.826} \\
\cline{2-12}
& \multirow{2}{*}{RMSE} & 200 & 14.875 & 18.783 & 9.258 & 9.917 & 10.212 & 10.175 & 8.480 & \textbf{6.015} & \textbf{6.881} \\
\cline{3-12}
& & 1000 & 20.432 & 14.586 & 9.712 & 9.800 & 10.004& 9.838 & 8.129 & \textbf{5.285} & \textbf{5.092} \\
\hline
\end{tabular}
}
\end{table*}

\paragraph{Heterogeneous treatment effect example} 
To verify that CBIPM also works well for heterogeneous treatment effects, we consider a new simulation model as follows.
We generate binary treatments $T_i$ from
$$ \mathbbm{P}(T_i=1 | \bm{X}_i) = \big( 1 + \text{exp}(-T_i^{\prime}) \big)^{-1}, $$
where $T_i ^{\prime} \sim \mathcal{N}(\mu_{i1}, 0.5)$,
\begin{align*}
\mu_{i1} \, =& \, \frac{\sin \big( \max(X_{i1}, X_{i2}, X_{i3}) \big)  +   \max(X_{i3}, X_{i4}, X_{i5})^2 }{2+(X_{i1}+X_{i5})^2} \\
&+ 4 X_{i1}^3 \sin (3 X_{i3}) \big( 1+\exp(X_{i4}-0.5 X_{i3} \big) + X_{i3}^2 + 2 (X_{i5}^2) \sin(X_{i4})  -3,
\end{align*}
and $\bm{X}_i = (X_{i1}, X_{i2}, X_{i3}, X_{i4}, X_{i5}, X_{i6})^\top \in \mathbbm{R}^6$ 
generating from $\text{Unif}(-2, 2)$. 
Also, we generate their corresponding outcomes from $Y_i \sim \mathcal{N}(\mu_{i2}, 0.1)$, where
$$ \mu_{i2} = 2(X_{i1}-2)^2 + 5 \cos(2 X_{i5}) \cdot \mathbbm{I}(T_i = 1)  + \frac{1}{X_{i2}^2+1} \cdot \frac{  \text{max}(X_{i1}, X_{i6})^3}{(1+2X_{i3}^2)} \cdot \sin(X_{i2}) + 3(2X_{i4}-1)^2.$$ 
Using Monte Carlo approximation with $10^5$ samples, we obtain the true values of the ATT and the ATE that are $-0.684$ and $-0.925$, respectively.

Table \ref{table_heter} presents the bias and RMSE for the ATT and the ATE estimators. 
For most setting, the N-CBIPMs outperform the other methods with large margins in terms of both the bias and RMSE.

\begin{table*}[h!]
\renewcommand{\arraystretch}{1.3}
\caption{\textbf{Heterogeneous treatment effect example} 
We generate 1000 simulations and report bias and RMSE.
For each dataset and measure, the two best values are marked in bold letters.
Similar to Table \ref{table_overlap}, we omit EB because its solution often does not converge.} \label{table_heter}
\centering
\scalebox{1.0}{\small
\begin{tabular}{|c|c|c|ccc|ccc|ccc|}
\hline
\multirow{2}{*}{Interest} & \multirow{2}{*}{Measure} &  \multirow{2}{*}{n} & \multicolumn{3}{c|}{Existing methods} & \multicolumn{3}{c|}{P-CBIPM} & \multicolumn{3}{c|}{N-CBIPM} \\
   \cline{4-12}
& &   & GLM & Boost & \multicolumn{1}{c|}{CBPS} & Wass & MMD & \multicolumn{1}{c|}{SIPM} & Wass & MMD & SIPM \\
    \hline \hline 
\multirow{4}{*}{ATT} & \multirow{2}{*}{Bias}  & 200 & -4.158 & -2.550 & -3.198 & -3.177 & -2.956 & -3.403 & -1.941 & \textbf{-0.033} & \textbf{-0.785}\\
\cline{3-12}
& & 1000 & -0.817 & -1.032 & -0.992 & -0.997 & -0.993 & -0.972 & \textbf{-0.318} & \textbf{0.375} & 0.987 \\
\cline{2-12}
& \multirow{2}{*}{RMSE} & 200 & 11.966 & 6.155 & 6.748 & 6.796 & 6.515 & 7.092 & \textbf{5.988} & \textbf{4.912} & 6.260 \\
\cline{3-12}
& & 1000 & 1.197 & 1.325 & 1.383 & 1.392 & 1.383 & 1.367 & \textbf{1.016} & \textbf{0.496} & 1.464 \\
\hline \hline
\multirow{4}{*}{ATE} & \multirow{2}{*}{Bias}  & 200 & -3.950 & -2.107 & -2.965 & -2.929 & -2.806 & -3.175 & -2.957 & \textbf{0.529} & \textbf{-1.046}  \\
\cline{3-12}
& & 1000 & -0.114 & -2.379 & -0.175 & -0.245 & -0.149 & -0.100 & \textbf{-0.033} & \textbf{0.044} & -0.196  \\
\cline{2-12}
& \multirow{2}{*}{RMSE} & 200 & 11.425 & \textbf{5.764} & 6.410 & 6.143 & 6.250 & 6.802 & 6.158 & \textbf{4.519} & 5.964 \\
\cline{3-12}
& & 1000 & 1.027 & 2.485 & 1.014 & 1.036 & 0.957 & \textbf{0.955} & 1.159  & \textbf{0.162} & 1.013 \\
\hline
\end{tabular}
}
\end{table*}

\newpage
\subsection{Semi-synthetic experiments} \label{App_semisyn}
We conduct semi-synthetic experiments using ACIC 2016 datasets and show the results in Table \ref{table_ACIC}. 
The ACIC 2016 datasets contain covariates, simulated treatment, and simulated response variables for the causal inference challenge in the 2016 Atlantic Causal Inference Conference \cite{dorie2019automated}. For each of 20 conditions, treatment and response data were simulated from real-world data corresponding to 4802 individuals and 58 covariates. Among 77 simulation settings, we select the last five ones and analyze 100 simulated data sets for each simulation setting. 

It is interesting to see that no IPM dominate others. While it works well for ATE, MMD is much inferior for ATT. On the other hand, SIPM is opposite (works well for ATT but not for ATE). Wasserstein IPM performs stably. The results indicate that the choice of the discriminator in the IPM is important for accurate estimation of the causal effect.

\begin{table*}[h!]
\caption{\textbf{ACIC 2016 datasets} 
For each dataset and measure, the two best values are marked in bold letters.
Similar to Table \ref{table_overlap}, we omit EB because its solution often does not converge.} \label{table_ACIC}
\begin{tabular}{|c|c|cc|ccc||cc|ccc|}
        \hline
        \multirow{3}{*}{Dataset} & \multirow{3}{*}{Measure} & \multicolumn{5}{c||}{ATT} & \multicolumn{5}{c|}{ATE}\\
        \cline{3-12}
         & & \multicolumn{2}{c|}{Existing methods} & \multicolumn{3}{c||}{N-CBIPM} & \multicolumn{2}{c|}{Existing methods} & \multicolumn{3}{c|}{N-CBIPM} \\
           \cline{3-12}
        & & GLM & \multicolumn{1}{c|}{CBPS} & Wass & MMD & \multicolumn{1}{c||}{SIPM} & GLM & \multicolumn{1}{c|}{CBPS} & Wass & MMD & \multicolumn{1}{c|}{SIPM} \\
        \hline \hline 
         \multirow{2}{*}{1} & \multirow{1}{*}{Bias} & 0.432 & 0.468 & \textbf{0.368} & 0.466 & \textbf{0.425} & 0.433 & 0.455 & \textbf{0.366} & \textbf{0.379} & 0.391\\
         \cline{2-12}
         & \multirow{1}{*}{RMSE} & 0.712 & 0.740 & \textbf{0.601} & 0.699 & \textbf{0.676} & 0.695 & 0.716 & \textbf{0.592} & \textbf{0.570} & 0.606\\
         \hline
         \multirow{2}{*}{2} & \multirow{1}{*}{Bias} & 0.100 & 0.102 & \textbf{0.095} & 0.128 & \textbf{0.095} & 0.109 & \textbf{0.056} & \textbf{0.098} & 0.107 & 0.106 \\
         \cline{2-12}
         & \multirow{1}{*}{RMSE} & 0.368 & 0.354 & \textbf{0.315} & 0.330 & \textbf{0.324} & 0.345 & 0.669 & \textbf{0.287} & \textbf{0.270} & 0.292\\
         \hline
         \multirow{2}{*}{3} & \multirow{1}{*}{Bias} & 0.248 & 0.272 & \textbf{0.219} & 0.253 & \textbf{0.234} &  0.223 & 0.252 & \textbf{0.192} & \textbf{0.189} & 0.202 \\
         \cline{2-12}
         & \multirow{1}{*}{RMSE} & 0.625 & 0.648 & \textbf{0.545} & 0.586 & \textbf{0.566} & 0.541 & 0.586 & \textbf{0.450} & \textbf{0.440} & 0.455 \\
         \hline
         \multirow{2}{*}{4} & \multirow{1}{*}{Bias} & 0.294 & 0.305 & \textbf{0.242} & 0.313 & \textbf{0.292} & 0.303 & 0.314 & \textbf{0.234} & \textbf{0.256} & 0.278 \\
         \cline{2-12}
         & \multirow{1}{*}{RMSE} & 0.534 & 0.526 & \textbf{0.435} & 0.501 & \textbf{0.490} & 0.494 & 0.505 & \textbf{0.402} & \textbf{0.397} & 0.412 \\
         \hline
         \multirow{2}{*}{5} & \multirow{1}{*}{Bias} & \textbf{0.340} & 0.421 & \textbf{0.355} & 0.406 & 0.388 & 0.358 & 0.363  & \textbf{0.285} & \textbf{0.290} & 0.302\\
         \cline{2-12}
         & \multirow{1}{*}{RMSE} & 0.783 & 0.812 & \textbf{0.649} & 0.724 & \textbf{0.717} & 0.679 & 0.708 & \textbf{0.562} & \textbf{0.550} & 0.570 \\
         \hline 
\end{tabular}
\end{table*}
\newpage

\subsection{Boxplots for experimental results in Section \ref{sec5.1}}

In Figure \ref{proposal_compare1}, we draw the boxplots of the estimated ATT obtained from the simulation in Section \ref{sec5.1} as the compliments to the
results of Table \ref{table1}.

\begin{figure}[h]
\centering
\subfigure[Linear $\operatorname{logit}(\pi(\cdot))$, $n=200$]{\includegraphics[width=0.45\linewidth]{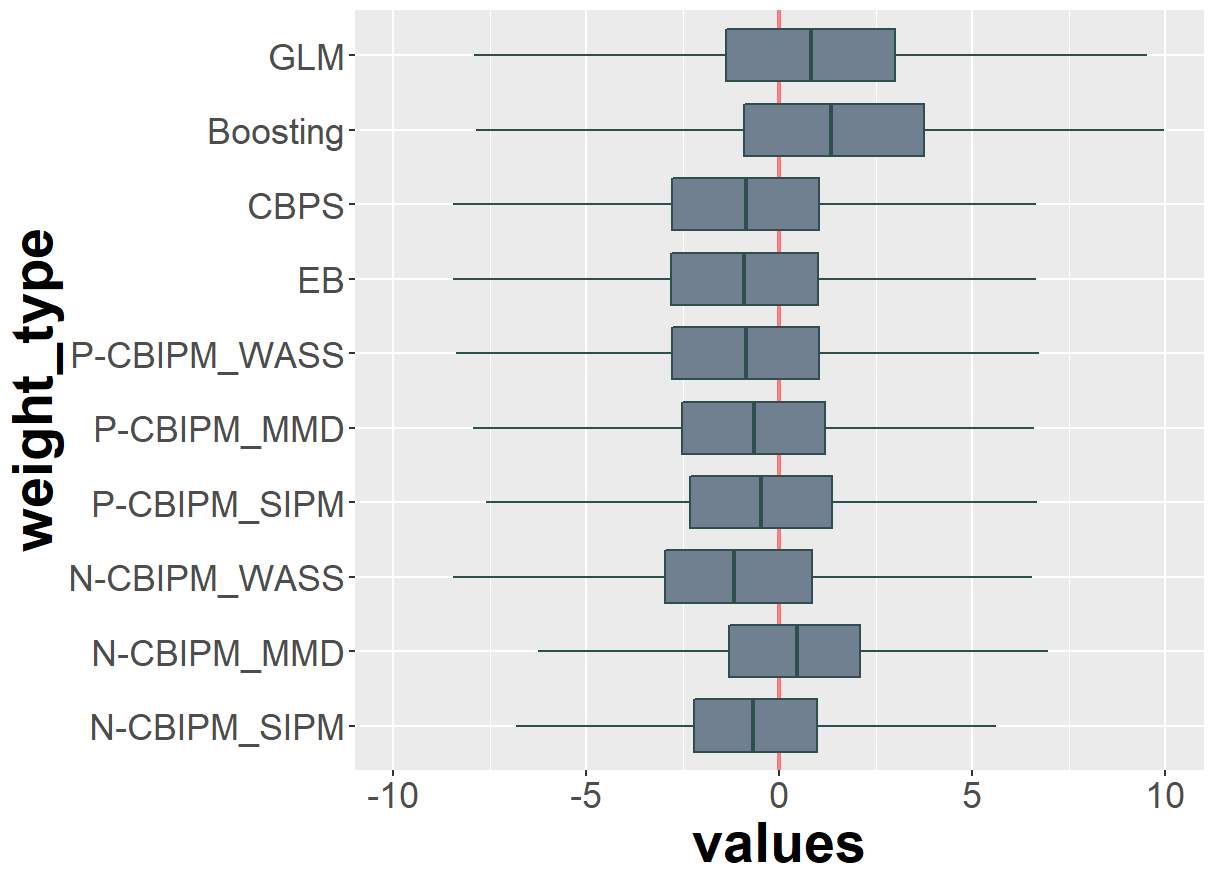}}
\hfill
\subfigure[Linear $\operatorname{logit}(\pi(\cdot))$, $n=1000$]{\includegraphics[width=0.45\linewidth]{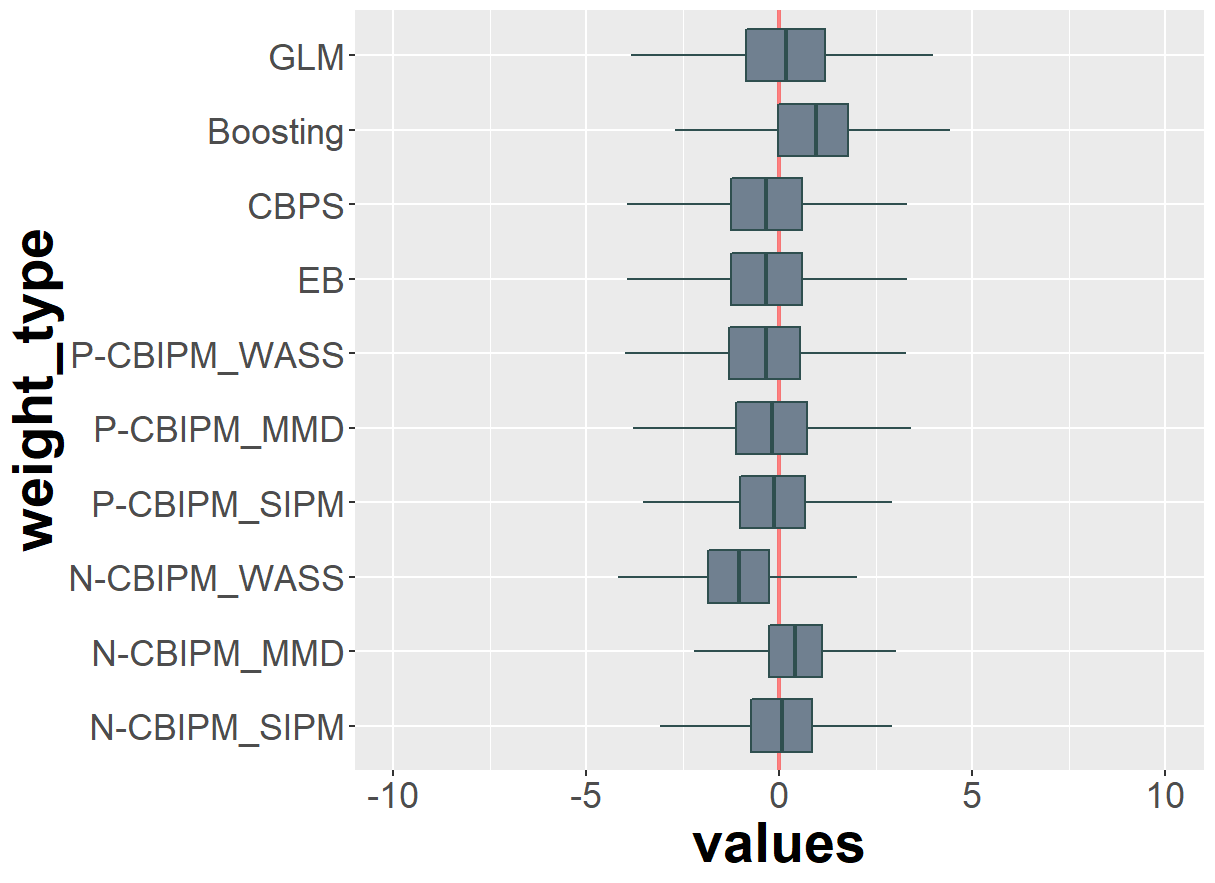}}
\hfill
\subfigure[Nonlinear $\operatorname{logit}(\pi(\cdot))$, $n=200$]{\includegraphics[width=0.45\linewidth]{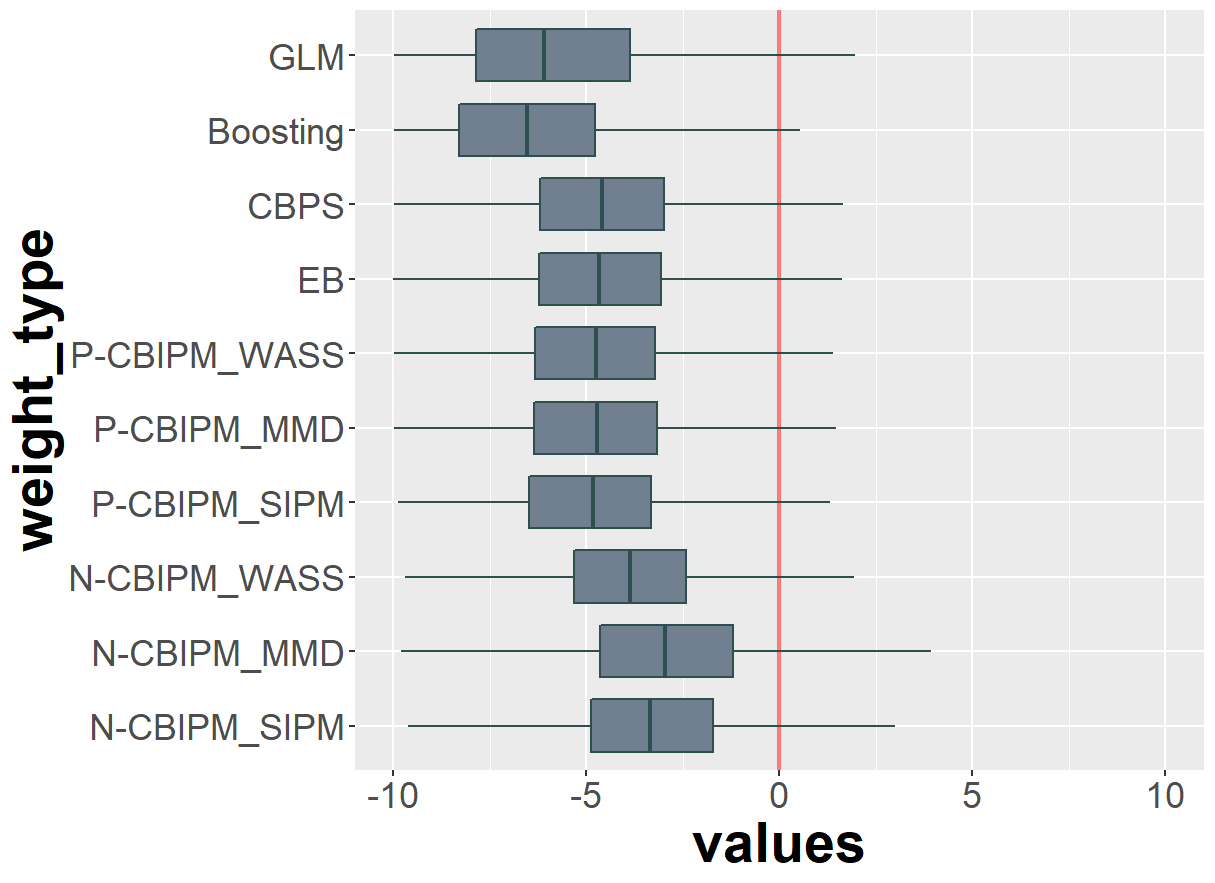}}
\hfill
\subfigure[Nonlinear $\operatorname{logit}(\pi(\cdot))$, $n=1000$]{\includegraphics[width=0.45\linewidth]{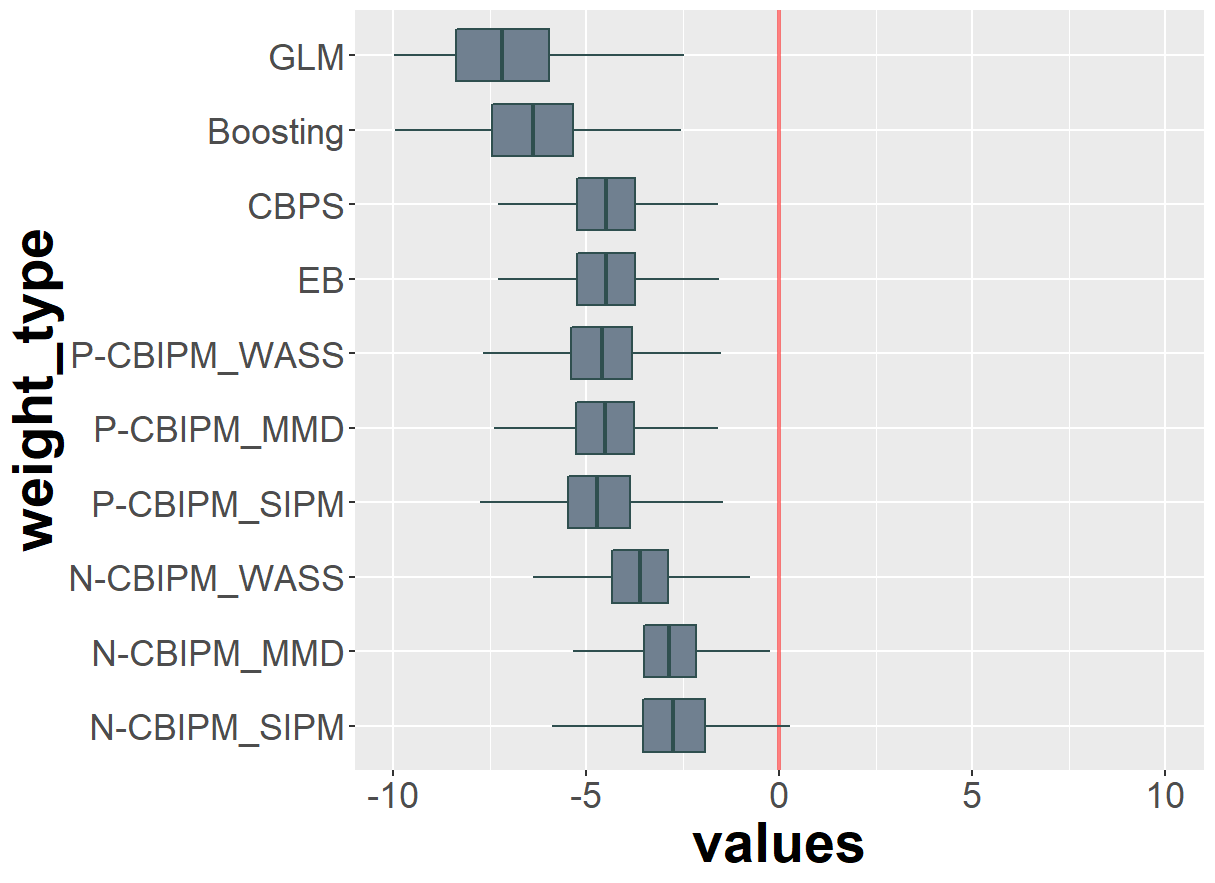}}
\caption{\textbf{Boxplots of the estimated ATT values on  Kang-Schafer example.} 
The boxplots are drawn based on the ATT estimates obtained from 1000 simulated datasets for the Kang-Schafer example. 
X-axis and Y-axis represent the estimated ATT values and the weighting methods, respectively.} \label{proposal_compare1}
\end{figure}

\subsection{Hypothesis test for experimental results in  Section \ref{sec5.2}} \label{App_realdata}

For the complements to Figure \ref{real_data}, we calculate the test statistics and corresponding p-values of the two sample Kolmogorov-Smirnov test between the (weighted) empirical distributions of the treated and control groups, whose results
are presented in Table \ref{table_KS}.
For most variables, especially for \textit{SENIORS}, N-CBIPM achieves better covariate balancing. 

\begin{table*}[h]
\renewcommand{\arraystretch}{1.3}
\caption{\textbf{Two-sample Kolmogorov-Smirnov test for STAR data} 
\label{table_KS}
We apply the two-sample Kolmogorov-Smirnov test to measure how well the weighting methods achieve covariate balancing
for STAR data.
For N-CBIPM, we use MMD.}
\centering
\scalebox{1.0}{\small
\begin{tabular}{|c|cccc|cccc|}
\hline
\multirow{2}{*}{Variables} & \multicolumn{4}{c|}{Test stat.} & \multicolumn{4}{c|}{p-value} \\
   \cline{2-9}
 & Eq.w & GLM & CBPS & N-CBIPM & Eq.w & GLM & CBPS & N-CBIPM \\
    \hline \hline 
ENRLMENT & 0.299 & 0.109 & 0.095 & \textbf{0.057} & 0.008 & 0.905 & 0.967 & \textbf{1.000} \\
\cline{1-9}
SENIORS & 0.332 & 0.139 & 0.122 & \textbf{0.073} & 0.002 & 0.671 & 0.823 & \textbf{1.000}  \\
\cline{1-9}
MNRTYPCT & 0.128 & 0.104 & 0.099 & \textbf{0.085} & 0.699 & 0.929 & 0.954 & \textbf{0.997} \\
\cline{1-9}
FRLCHPCT & 0.192 & 0.073 & \textbf{0.068} & 0.089 & 0.205 & 0.999 & \textbf{1.000} & 0.995 \\
\cline{1-9}
\hline
\end{tabular}
}
\end{table*}

\newpage

\newpage
\subsection{Abolation study : the number of ensemble models in the SIPM} \label{sec_nsipm}

In Appendix \ref{app_C}, we propose to use an ensemble technique for the SIPM to avoid model collapse.
To illustrate the efficiency of the ensemble techniques, we investigate the accuracy (RMSE) and computing time
of the ensemble SIPM algorithm with the various numbers of ensemble models, whose results are presented
in Figure \ref{proposal_compare2}. The ensemble technique improves
the accuracies significantly without increasing computing time much. 

\begin{figure}[h]
\centering
\subfigure[P-CBIPM with SIPM, $n=200$]{\includegraphics[width=0.45\linewidth]{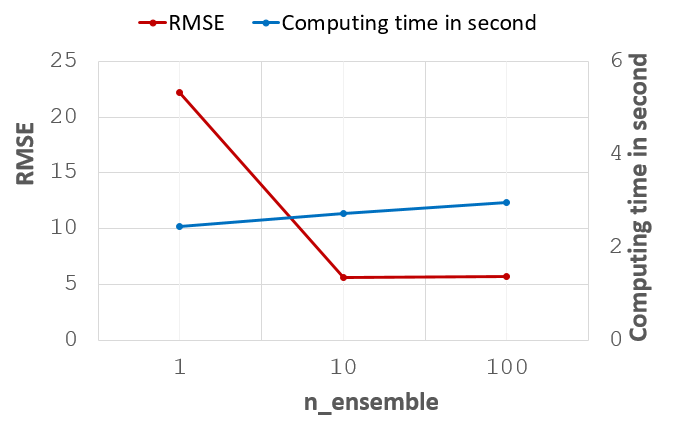}}
\hfill
\subfigure[P-CBIPM with SIPM, $n=1000$]{\includegraphics[width=0.45\linewidth]{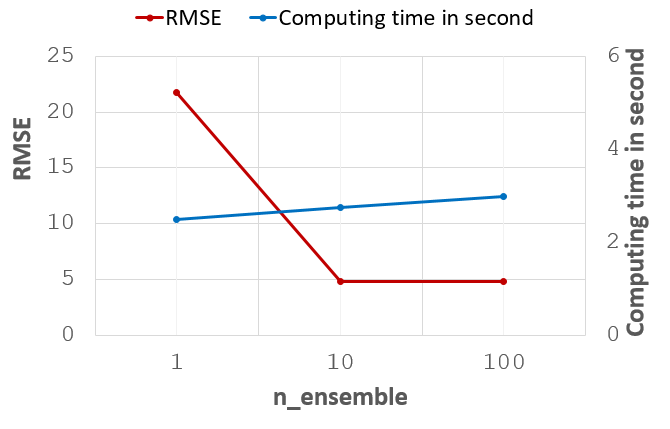}}
\hfill
\subfigure[N-CBIPM with SIPM, $n=200$]{\includegraphics[width=0.45\linewidth]{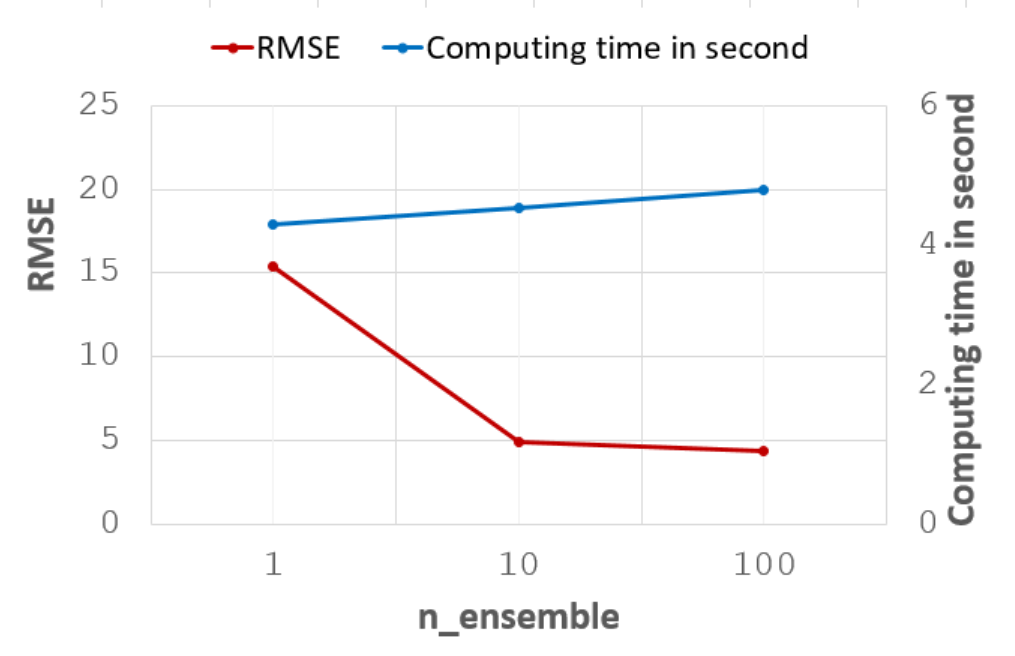}}
\hfill
\subfigure[N-CBIPM with SIPM, $n=1000$]{\includegraphics[width=0.45\linewidth]{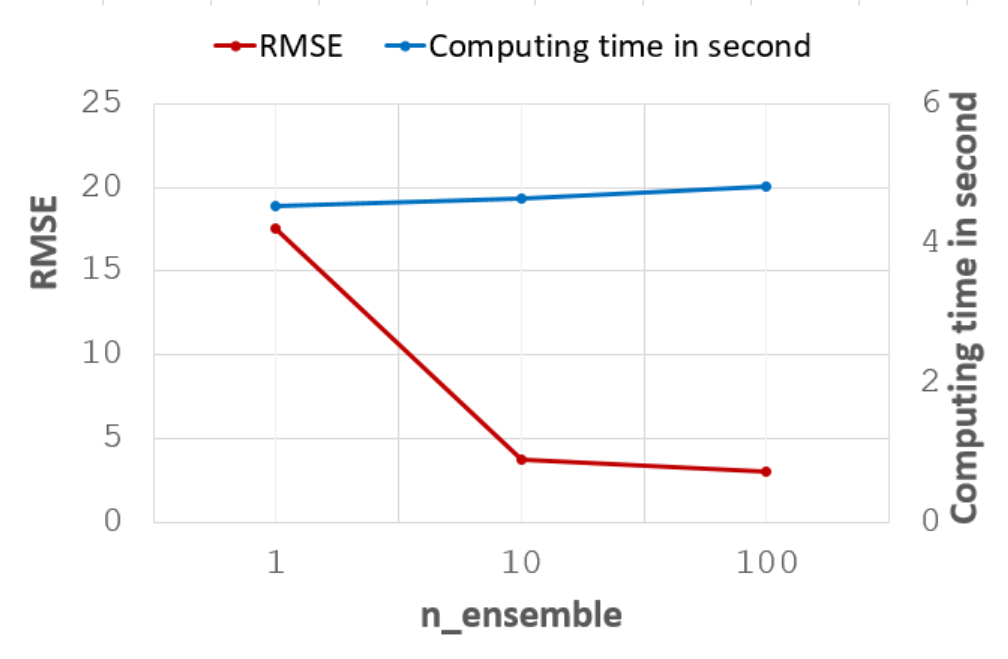}}
\caption{\textbf{RMSE and Computing time for ensemble SIPM.}
Accuracies and computing times of the P-CBIPM and N-CBIPM with the SIPM according to the number of the ensembles are compared for the Kang-Schafer example considered in Section \ref{sec5.1}.} \label{proposal_compare2}
\end{figure}

\newpage
\section{Additional proofs for the manuscript}
\renewcommand{\theequation}{F.\arabic{equation}}

\subsection{Unbiasedness of CBPS when true outcome model is linear} \label{CBPS_unbiased}
For $\bm{\phi}\left(\bm{X}\right) = \bm{X}$, constraint (\ref{eq:CBPS-eq}) becomes
\begin{align*}
\frac{1}{n} \sum_{i : T_i = 0} \frac{\pi_{\hat{\bm{\beta}}} (\bm{X}_i)}{1-\pi_{\hat{\bm{\beta}}} (\bm{X}_i) } \bm{X}_i = \frac{1}{n} \sum_{i : T_i = 1} \bm{X}_i.
\end{align*}
Hence, if $m_0 ( \cdot) : \mathbb{R}^d \to \mathbb{R}$ is a linear function, we get
\begin{align*}
\sum_{i : T_i = 0} \frac{\pi_{\hat{\bm{\beta}}} (\bm{X}_i)}{1-\pi_{\hat{\bm{\beta}}} (\bm{X}_i) } m_0(\bm{X}_i) = \sum_{i : T_i = 1} m_0(\bm{X}_i).
\end{align*}            
Hence, for given $\mathcal{C}^{(n)} = \{(\bm{X}_i, T_i)\}_{i=1}^n$, we obtain
\begin{align*}
    &\mathbbm{E}(\mathbbm{E}(\sum_{i : T_i = 1} \frac{1}{n_1} Y_i- \sum_{i : T_i = 0} \frac{1}{n_1}\frac{\pi_{\hat{\bm{\beta}}} (\bm{X}_i)}{1-\pi_{\hat{\bm{\beta}}} (\bm{X}_i)} Y_i \mid \mathcal{C}^{(n)}))\\
    =& \mathbbm{E}( \frac{1}{n_1}\sum_{i : T_i = 1} m_1(\bm{X}_i)- \frac{1}{n_1} \sum_{i : T_i = 0} \frac{\pi_{\hat{\bm{\beta}}} (\bm{X}_i)}{1-\pi_{\hat{\bm{\beta}}} (\bm{X}_i)} m_0(\bm{X}_i) ) \\
    =& \mathbbm{E}( \frac{1}{n_1} \sum_{i : T_i = 1} (m_1(\bm{X}_i)-  m_0(\bm{X}_i) )) \\
    =& \mathbbm{E}(\operatorname{SATT}) = \operatorname{ATT}.
\end{align*}

\subsection{Derivation for error decomposition} \label{decom_proof}
We obtain (\ref{decom}) by
\begin{align*}
    &\widehat{\operatorname{ATT}}^{\bm{w}} - \operatorname{ATT}\\
    =& (\sum_{i : T_i = 1}   \frac{1}{n_1}Y_i- \sum_{i : T_i = 0} w_i Y_i)
    - \operatorname{SATT} + (\operatorname{SATT} - \operatorname{ATT})\\
    =& (\sum_{i : T_i = 1}   \frac{1}{n_1}Y_i- \sum_{i : T_i = 0} w_i Y_i)
    - \frac{1}{n_1}\sum_{i : T_i = 1} ( m_1 (\bm{X}_i) - m_0 (\bm{X}_i)) + (\operatorname{SATT} - \operatorname{ATT})\\
    =& \frac{1}{n_1}\sum_{i : T_i = 1} m_0 (\bm{X}_i) + \frac{1}{n_1}\sum_{i : T_i = 1} \frac{Y_i - m_1(\bm{X}_i)}{n_1} - \sum_{i : T_i = 0} w_i Y_i + (\operatorname{SATT} - \operatorname{ATT})\\
    =& \frac{1}{n_1}\sum_{i : T_i = 1} m_0 (\bm{X}_i) - \sum_{i : T_i = 0} w_i m_0 (\bm{X}_i) + \frac{1}{n_1}\sum_{i : T_i = 1} \frac{Y_i - m_1(\bm{X}_i)}{n_1} - \sum_{i : T_i = 0} w_i (Y_i - m_0 (\bm{X}_i)) + (\operatorname{SATT} - \operatorname{ATT}) \\
    =& \operatorname{err}_{\text{bal}}^{\bm{w}} + \operatorname{err}_{\text{obs}}^{\bm{w}} + (\operatorname{SATT} - \operatorname{ATT}).
\end{align*}
\subsection{CBPS as the special case of P-CBIPM} \label{CBPS_equ}
Consider solving (\ref{solve_P_CBIPM}) over
\begin{align*}
    f(\bm{x} ; \bm{\theta} ) &=  \bm{\theta}^{\top} \bm{x}, \qquad \bm{\theta} \in \mathbb{R}^d,\\
    \mathcal{M}^{\text{linear}} &= \left\{m(\cdot) : m(\bm x) = \bm{\alpha}^\top \bm{x}, \bm \alpha \in \mathbb{R}^d, ||\bm{\alpha}||_{\infty} \leq 1 \right\}.            
\end{align*}
Then, this formulation of P-CBIPM is indeed the same as that of CBPS. 
More specific,  
\begin{align*}
d_{\mathcal{M}^{\text{linear}}}(\mathbbm{P}_{0,n}^{\bm{w}}, \mathbbm{P}_{1,n} ) =& \sup_{m\in \mathcal{M}^{\text{linear}}} \left|  \sum_{i : T_i = 0} w_i m (\bm{X}_i) - \sum_{i : T_i = 1} \frac{m (\bm{X}_i)}{n_1} \right| \nonumber\\
=& \sup_{||\bm{\alpha}||_{\infty} \leq 1} \left|  \sum_{i : T_i = 0} w_i \bm{\alpha}^\top \bm{X}_i - \sum_{i : T_i = 1} \frac{\bm{\alpha}^\top \bm{X}_i }{n_1} \right|\\
=& \sup_{||\bm{\alpha}||_{\infty} \leq 1} \left| \sum_{i : T_i = 0} \sum_{j=1}^d w_i \alpha_j X_{ij} - \sum_{i : T_i = 1} \sum_{j=1}^d \frac{ \alpha_j X_{ij} }{n_1}  \right|\\
=& \sup_{||\bm{\alpha}||_{\infty} \leq 1} \left| \sum_{j=1}^d \alpha_j \left( \sum_{i : T_i = 0} w_i X_{ij} - \sum_{i : T_i = 1} \frac{ X_{ij} }{n_1} \right) \right|\\
=& \sum_{j=1}^d \left|  \sum_{i : T_i = 0} w_i X_{ij} - \sum_{i : T_i = 1} \frac{ X_{ij} }{n_1} \right|,
\end{align*}     
where $\bm{X}_i = (X_{i1}, \dots, X_{id})^{\top}$.
Since $w_i^{\bm{\theta}}$ is expressed as
\begin{align*}
     w_i^{\bm{\theta}} := \frac{\mathbb{I}(T_i = 0) \exp(\bm{\theta}^\top \bm{X}_i )}{ \sum_{i : T_i = 0} \exp({\bm{\theta}^\top \bm{X}_i})} && i \in [n]
\end{align*}
for given $\bm{\theta} \in \mathbb{R}^d$, $\bm{w}$ has a degree of freedom of $d$.           
Hence, there exists a unique solution $\widehat{\bm{\theta}}_n$ such that $d_{\mathcal{M}^{\text{linear}}}(\mathbbm{P}_{0,n}^{\bm{w}^{\widehat{\bm{\theta}}_n}}, \mathbbm{P}_{1,n})$ equals zero, which is identical to the solution of CBPS.

\end{document}